\documentclass{article}

\IfFileExists{PRIMEarxiv.sty}{\usepackage{PRIMEarxiv}}{}
\usepackage[utf8]{inputenc}
\usepackage[T1]{fontenc}
\usepackage{hyperref}
\usepackage{url}
\usepackage{booktabs}
\usepackage{amsfonts}
\usepackage{amsmath}
\usepackage{amsthm}
\usepackage{amssymb}
\usepackage{xcolor}
\usepackage{multirow}
\usepackage{graphicx}
\usepackage{fancyhdr}
\usepackage{array}
\usepackage{caption}
\graphicspath{{media/}{figure/}}

\newtheorem{theorem}{Theorem}

\newtheorem{definition}{Definition}

\providecommand{\keywords}[1]{\par\bigskip\noindent\textbf{Keywords: } #1}

\title{A Polynomial Architecture–Attribution Co-Design Framework for Exact Aumann–Shapley Attribution in GNNs}

\author{
  \textbf{
    Bizu Feng\textsuperscript{1,2,3},
    Zhimu Yang\textsuperscript{4},
    Shuming Wang\textsuperscript{2,3},
    Shaode Yu\textsuperscript{4}
  }
  \\
  \textbf{
    Yuan Cheng\textsuperscript{1,2},
    Xiaojun Qian\textsuperscript{1},
    Zixin Hu\textsuperscript{1,2,*}
  }
  \\[0.7em]
  {\small
    \textsuperscript{1}Institute of Artificial Intelligence Innovation and Industry,
    Fudan University, Shanghai, China
  }
  \\
  {\small
    \textsuperscript{2}Shanghai Academy of AI for Science,
    Shanghai, China
  }
  \\
  {\small
    \textsuperscript{3}Human Phenome Institute,
    Fudan University, Shanghai, China
  }
  \\
  {\small
    \textsuperscript{4}School of Information and Communication Engineering,
    Communication University of China, Beijing, China
  }
  \\[0.5em]
  {\footnotesize
    \texttt{bzfeng25@m.fudan.edu.cn},
    \texttt{muyuzhierchengse@gmail.com}
  }
  \\
  {\footnotesize
    \texttt{smwang@m.fudan.edu.cn},
    \texttt{yushaodecuc@cuc.edu.cn}
  }
  \\
  {\footnotesize
    \texttt{cheng\_yuan@fudan.edu.cn},
    \texttt{qianxiaojun@fudan.edu.cn}
  }
  \\[0.4em]
  {\small
  \textsuperscript{*}Correspondence:
  \texttt{huzixin@fudan.edu.cn}
  }
}

\begin{document}
\maketitle

\begin{abstract}
We study feature-level and node-level explanations for graph neural networks (GNNs) through the lens of Aumann–Shapley attribution. Path-integral methods such as Integrated Gradients provide an axiomatic formulation of attribution, but their practical use in deep GNNs typically relies on finite-sample numerical approximations to the path integral, requiring a trade-off between quadrature error and computational cost. This paper proposes APEX, a model-attribution co-design framework that makes the attribution integral exactly computable under a polynomial GNN architecture. The key component is PolyGIN, a GIN-style graph network whose message-passing, normalization, and transformation operations preserve a bounded multivariate polynomial form for scalar model scores, such as pre-softmax logits. We show that, for a PolyGIN with $L$ polynomial transformation blocks, the derivative along the attribution path has degree at most $2^L-1$. Therefore, Gauss--Legendre quadrature can evaluate the Aumann--Shapley path integral exactly, up to floating-point precision, with $2^{L-1}$ deterministic evaluation points. The resulting attributions can be computed at the feature level and then aggregated into node-level scores while preserving completeness. Experiments on synthetic and real-world graph benchmarks show that PolyGIN maintains competitive predictive performance, while the complete APEX framework achieves higher attribution fidelity than the compared baselines and substantially reduces the number of evaluations required for path integration.
\end{abstract}


\section{Introduction}

Graph neural networks (GNNs) have become a standard modeling tool for graph-structured data, with applications in social networks, recommendation systems\cite{wu2022graph, fan2019graph}, natural language processing\cite{yao2019graph}, and molecular property prediction \cite{jiang2021could}. Their success largely comes from iterative message passing\cite{gilmer2017neural}, which allows node representations to combine local features with multi-hop neighborhood information. However, the same mechanism also makes their predictions difficult to interpret: after several rounds of nonlinear aggregation, it is often unclear which input features, nodes, or substructures are responsible for a model output. This lack of transparency limits the use of GNNs in settings where reliable explanations are needed.

A common way to explain neural predictions is to assign attributions to input variables. Among attribution methods, Shapley-style values are attractive because they are connected to explicit axioms of fair contribution allocation. In differentiable models, continuous contribution allocation can be formulated through path methods related to the Aumann–Shapley construction. Integrated Gradients instantiates such attribution along the straight-line path from a baseline to the input\cite{sundararajan2017axiomatic, aumann2015values, sundararajan2020many}. These methods integrate the gradient of the model along a path from a baseline input to the target input. In principle, this provides a rigorous attribution rule. In practice, however, the integral is rarely available in closed form for standard GNNs, whose nonlinear activations and normalization operations make the integrand analytically difficult to handle. Figure~\ref{fig:motivation} illustrates this limitation and the central motivation of APEX: conventional path attribution relies on a user-specified sampling resolution, whereas APEX uses an architecture-induced degree bound to determine a finite quadrature budget that eliminates truncation error under the stated conditions.Existing implementations therefore rely on numerical approximations such as Riemann sums. With few integration points, the attribution may suffer from non-negligible quadrature error; with many points, the computation becomes expensive because each point requires model evaluation and gradient computation.

This paper addresses the above limitation from a model-attribution co-design perspective. Rather than applying a numerical attribution method to an arbitrary nonlinear GNN, we ask whether the GNN architecture can be designed so that the attribution integral becomes exactly computable. Our answer is APEX, a framework built on a polynomial GNN architecture called PolyGIN. PolyGIN replaces non-polynomial components, such as ReLU activations and standard normalization layers, with polynomial or linear operations. As a result, the model output remains a multivariate polynomial in the input node features. This algebraic structure gives a finite degree bound for the derivative along the Aumann--Shapley integration path. Once this degree bound is known, Gauss--Legendre quadrature provides an exact finite-sum evaluation of the path integral. Throughout this paper, APEX refers to the complete co-designed framework consisting of the PolyGIN predictive architecture and its architecture-certified exact attribution procedure, rather than to one aspect alone.

\begin{figure}[t]
    \centering
    \IfFileExists{figure/motivation_v2.pdf}{%
        \includegraphics[width=0.9\linewidth]{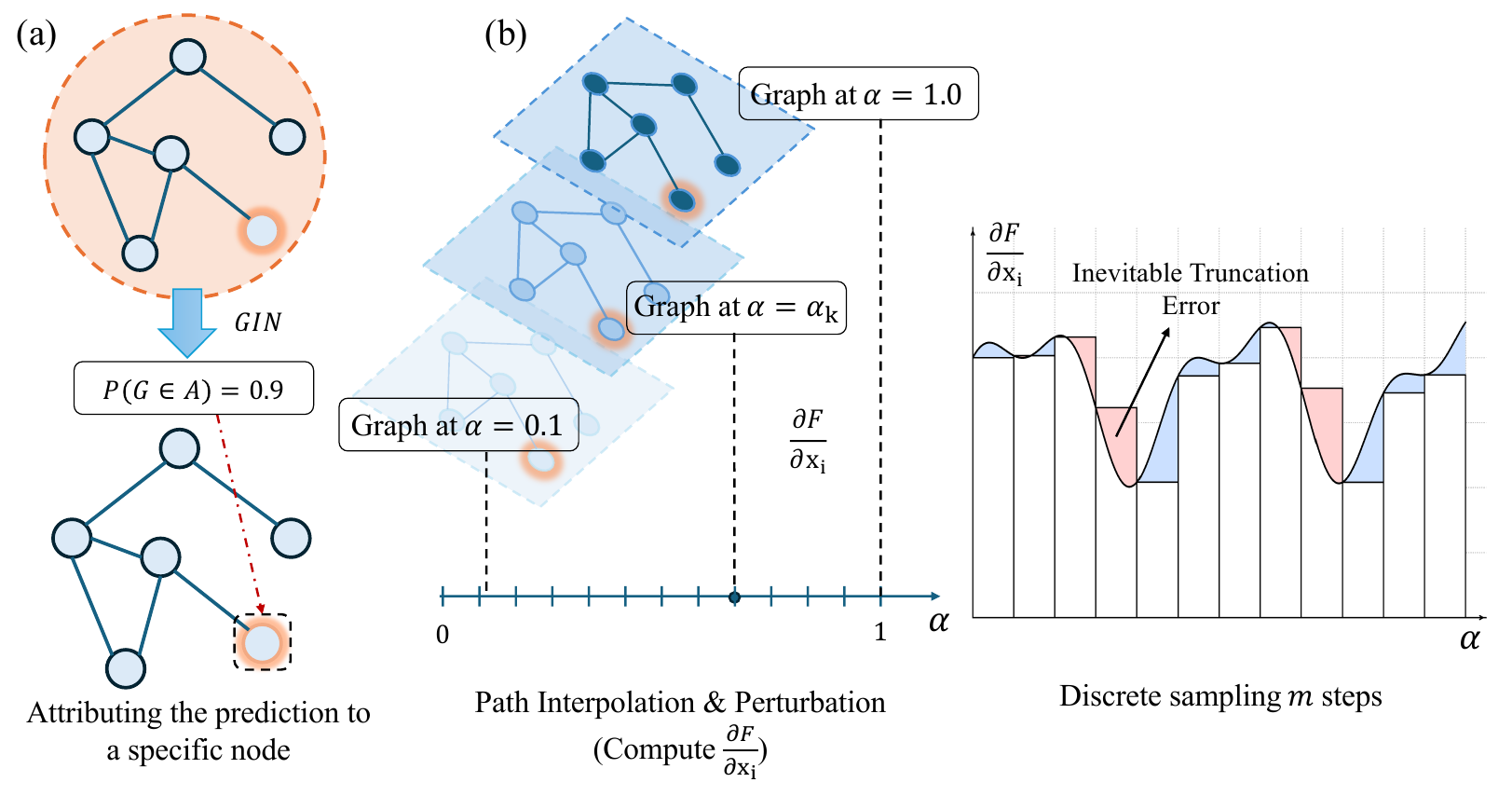}%
    }{%
        \fbox{\parbox{0.74\linewidth}{\centering Placeholder for the motivation figure. The final figure should compare numerical path integration with the exact quadrature used by APEX.}}%
    }
    \caption{Motivation of APEX. Standard path-integral attribution methods approximate the integral with numerical sampling, which introduces a trade-off between quadrature error and computation. Under the proposed polynomial architecture, the integrand has a known degree bound, allowing Gauss--Legendre quadrature to evaluate the integral exactly with a fixed number of points.}
    \label{fig:motivation}
\end{figure}

The main idea is that if the model output is a polynomial of degree at most $2^L$ in the input features, then the partial derivative involved in the attribution integral has degree at most $2^L-1$ along the straight-line path from baseline to input. Since an $m$-point Gauss--Legendre rule integrates all polynomials of degree at most $2m-1$ exactly, setting $m=2^{L-1}$ is sufficient to remove quadrature error under the proposed architecture. APEX therefore computes Aumann--Shapley attributions using a fixed, deterministic set of evaluation points rather than heuristic sampling. With respect to the numerical integration resolution, conventional numerical path integration requires $\Theta(N_{\mathrm{step}})$ gradient evaluations, whereas APEX uses an architecture-determined budget of exactly $2^{L-1}$ deterministic evaluations. Therefore, for a fixed-depth architecture, the attribution budget is independent of the user-specified integration resolution $N_{\mathrm{step}}$.

Our contributions are as follows:
\begin{itemize}
    \item We propose PolyGIN, a polynomial variant of the Graph Isomorphism Network that preserves a bounded multivariate polynomial mapping from input node features to model outputs.
    \item We prove a degree bound for PolyGIN and show that the derivative along the Aumann--Shapley path has degree at most $2^L-1$, where $L$ is the number of polynomial transformation blocks.
    \item We propose APEX, a GNN-specific architecture--attribution
    co-design framework that transforms GNN path attribution from a sampling-resolution-dependent numerical approximation problem into a finite exact computation. PolyGIN guarantees that the path integrand has degree at most $2^L-1$, which yields a sufficient $2^{L-1}$ point Gauss--Legendre rule for exact integration. APEX therefore eliminates quadrature truncation error under exact arithmetic, without user-specified integration resolutions or empirical convergence sweeps, whereas conventional numerical path integration requires $\Theta(N_{\mathrm{step}})$ gradient evaluations with accuracy dependent on the chosen $N_{\mathrm{step}}$.
    \item We extend feature-level attributions to node-level explanations through additive aggregation while preserving completeness, and evaluate the resulting explanations on five graph benchmarks in terms of predictive performance, fidelity, efficiency , and exactness.
\end{itemize}

\section{Related Work}
\label{sec:related}

\subsection{Graph Neural Network Explanations and Axiomatic Path Attribution}
Research on GNN explainability spans multiple paradigms, ranging from the post-hoc identification of influential graph components to the axiomatic attribution of model outputs to input variables\cite{yuan2022explainability}. One major line of post-hoc research focuses on identifying graph components that are influential to a trained model’s prediction. GNNExplainer learns instance-specific masks over graph structure and node features to extract a compact explanatory subgraph\cite{ying2019gnnexplainer}, whereas PGExplainer amortizes this process by learning a parameterized distribution over explanatory edges that can generalize across instances\cite{luo2020parameterized}. SubgraphX formulates explanation as a search over candidate subgraphs and uses Monte Carlo tree search together with approximate Shapley-value evaluations to identify predictive structures\cite{yuan2021explainability}. FlowX changes the explanatory unit from individual nodes or edges to message flows across GNN layers and estimates their relevance through a sampling- and learning-based procedure\cite{gui2023flowx}. Gradient-based techniques, including saliency and Grad-CAM-style methods adapted to graph models, instead obtain node- or feature-level relevance from local derivatives and intermediate representations\cite{pope2019explainability}. These methods provide useful mechanisms for selecting explanatory structures or ranking graph components, but their scores are defined through method-specific masking objectives, combinatorial search, learned explainers, or local sensitivity measures. However, these methods are generally not designed to satisfy the completeness constraint required by axiomatic path attribution and therefore do not provide a corresponding completeness guarantee.

Shapley-style attribution provides a complementary perspective by defining feature contributions through explicit allocation principles. The Shapley value characterizes contribution allocation in finite cooperative games\cite{shapley1953value}, while the Aumann--Shapley construction extends related principles to continuous settings\cite{aumann2015values}. For differentiable models, path attribution methods distribute the output difference between a reference input and a target input by integrating directional derivatives along a connecting path. Integrated Gradients instantiates this construction using the straight-line path from a baseline to the input and satisfies properties including sensitivity and implementation invariance\cite{sundararajan2017axiomatic}. When the path integral is evaluated exactly, the resulting feature attributions also satisfy completeness: their sum equals the difference between the model scores at the input and the baseline.

For standard neural networks, however, the path integrand generally lacks an analytically usable form. Practical implementations therefore replace the integral with a finite numerical rule, commonly using uniformly spaced samples or related composite formulas. The resulting attribution depends not only on the model and baseline but also on the selected integration budget. Increasing this budget reduces numerical error at the cost of additional forward and backward evaluations, whereas an insufficient budget can leave a nonzero completeness gap and alter the numerical values used for feature or node ranking. This issue is distinct from the combinatorial approximation used by Shapley-based graph explainers. For example, SubgraphX approximates coalition-based contributions of subgraphs, and FlowX estimates contributions associated with message flows\cite{yuan2021explainability,gui2023flowx}. APEX instead considers continuous Aumann--Shapley attribution of input features and derives node-level explanations by additive aggregation. Its objective is not to introduce another sampling strategy, but to construct a model class for which the continuous path integral has a known finite exact evaluation budget.

\subsection{Architectural Conditions for Exact Path Attribution}

Most attribution methods treat the predictive model as fixed and subsequently approximate the explanation required for that model. A smaller line of research has shown that architectural structure can make axiomatic attribution substantially more efficient. Hesse et al.\ prove that bias-free, nonnegatively homogeneous neural networks admit an axiomatic attribution that can be evaluated with a single forward and backward pass\cite{hesse2021fast}. Their result exploits the scaling identity of positively homogeneous functions and is naturally associated with attribution from the zero reference. It provides an important example of model--attribution compatibility, but the reduction does not directly extend to general baselines or to nonhomogeneous transformations containing additive biases and higher-order interactions.

APEX follows the same broad principle that model structure can simplify attribution, but relies on a different algebraic condition. Rather than requiring the path integral to collapse to a single endpoint-gradient computation, APEX constrains the scalar input-to-logit mapping to have bounded polynomial degree. Along an arbitrary straight-line path, the corresponding derivative is then a univariate polynomial whose degree is determined by the architecture. This degree certificate yields a sufficient Gauss--Legendre evaluation budget without selecting a sampling resolution empirically or specifying an error tolerance. The resulting guarantee covers nonhomogeneous polynomial interactions and general reference inputs, while imposing its own explicit restriction: every operation affecting the attributed scalar score must preserve polynomial dependence on the input features.

\subsection{Polynomial Architectures for Exact Attribution}

Polynomial neural networks have previously been studied as expressive model classes rather than as mechanisms for exact attribution. Kileel et al.\ characterize the function spaces induced by deep networks with polynomial activations and analyze how architectural choices affect their expressive power\cite{kileel2019expressive}. $\Pi$-Nets construct high-order input polynomials through structured skip connections and tensorized parameterizations, demonstrating that polynomial architectures can represent nontrivial higher-order interactions rather than reducing to linear models\cite{chrysos2020p}. These results motivate polynomial networks as viable predictive architectures, but they do not use the degree of the input-to-output map to determine an exact path-attribution computation. In APEX, the polynomial restriction is introduced for this specific purpose: the architecture exposes a finite degree bound, and the attribution algorithm uses that bound to select a deterministic quadrature budget.

This use of polynomial structure should also be distinguished from polynomial graph filters. Spectral architectures such as Chebyshev graph convolution express the propagation operator as a polynomial in the graph Laplacian\cite{defferrard2016convolutional}. Polynomial dependence on a graph operator does not imply that the complete network is polynomial in its input node features, since subsequent activations, normalization layers, attention mechanisms, or output transformations may remain non-polynomial. APEX instead requires the complete scalar score used for attribution, such as a pre-softmax logit, to be a bounded-degree polynomial of the attributed inputs. PolyGIN realizes this condition while retaining GIN-style neighborhood aggregation\cite{xu2018how}. The central distinction from prior polynomial architectures and graph filters is therefore not the use of polynomials itself, but the architecture--attribution co-design that converts a model-level degree bound into an exact and automatically determined Aumann--Shapley evaluation budget.

\section{Methodology}
\label{sec:methodology}

\begin{figure}[h]
    \centering
    \IfFileExists{figure/workflow_v2.pdf}{%
        \includegraphics[width=1.0\linewidth]{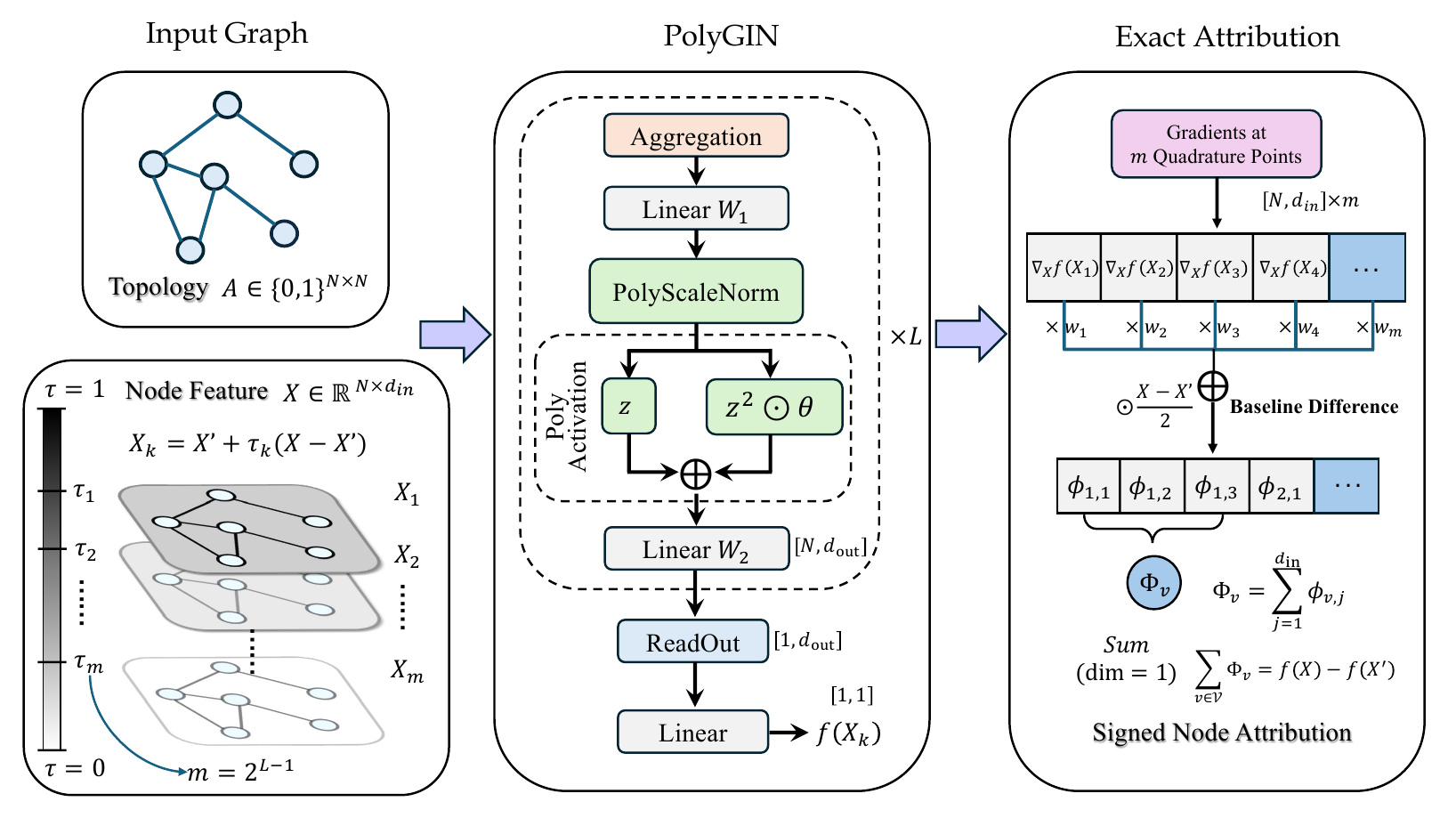}%
    }{%
        \fbox{\parbox{0.78\linewidth}{\centering Placeholder for the method figure. The final figure should illustrate PolyGIN, the polynomial-degree bound, and the Gauss--Legendre evaluation points used by APEX.}}%
    }
    \caption{Overview of the APEX computation. PolyGIN constrains the model output to a bounded polynomial space. This gives a degree bound for the derivative along the attribution path, allowing APEX to replace numerical path sampling with exact Gauss--Legendre quadrature.}
    \label{fig:overview}
\end{figure}

This section presents the APEX framework. As summarized in Figure~\ref{fig:overview}, PolyGIN constrains the input-to-logit mapping to a bounded polynomial space, whose degree bound determines a sufficient finite evaluation budget for exact Gauss--Legendre quadrature. The resulting feature-level attributions are then additively aggregated into signed node-level explanations while preserving completeness. We first formulate the attribution problem, then define the PolyGIN architecture, establish its degree bound, derive the exact quadrature rule, and present the node-level aggregation.

\subsection{Problem Formulation}
Let $f$ be a trained graph model and let $X\in\mathbb{R}^{N\times d}$ be the input node-feature matrix. For simplicity, we write the flattened input features as $x_1,\ldots,x_n$, where $n=Nd$. In classification tasks, $f$ denotes the scalar score of the target class, implemented as the pre-softmax logit rather than the post-softmax probability. This distinction is important because a softmax layer would introduce non-polynomial operations and would therefore fall outside the exactness guarantee of APEX.

Given a baseline input $X'$, the straight-line attribution path is
\begin{equation}
    \gamma(\tau)=X' + \tau(X-X'), \qquad \tau\in[0,1].
\end{equation}
For a target output $f(X)$, the Aumann--Shapley attribution of feature $x_i$ is
\begin{equation}
\label{eq:as}
    \phi_i = (x_i-x_i')\int_0^1 \frac{\partial f(\gamma(\tau))}{\partial x_i}\,d\tau.
\end{equation}
This is the straight-line path attribution used by Integrated Gradients \cite{sundararajan2017axiomatic}. The main computational challenge is evaluating the integral in Eq.~\eqref{eq:as}. Standard IG approximates it with a finite numerical sum. APEX instead designs $f$ so that the integrand is a bounded-degree polynomial in $\tau$, allowing exact quadrature.

\subsection{PolyGIN: A Polynomial Graph Architecture}
PolyGIN follows the neighborhood aggregation structure of GIN but replaces the nonlinear transformation module with polynomial and linear operations. This ensures that the model output remains a multivariate polynomial in the input node features.

\begin{definition}[PolyActivation]
Given a vector $z$, the polynomial activation is
\begin{equation}
    \sigma_{\mathrm{poly}}(z)=z+\theta\odot z^2,
\end{equation}
where $\theta$ is a learnable parameter vector, $\odot$ denotes element-wise multiplication, and $z^2$ is applied element-wise.
\end{definition}

The linear term preserves an identity path, while the quadratic term introduces nonlinear feature interactions. Since the operation is polynomial, it preserves algebraic tractability.

\begin{definition}[PolyScaleNorm]
Given a vector $z$, PolyScaleNorm is defined as
\begin{equation}
    \mathrm{Norm}_{\mathrm{poly}}(z)=s\odot z,
\end{equation}
where $s$ is a learnable dimension-wise scaling vector.
\end{definition}

A PolyGIN layer first performs the injective sum-aggregation form introduced by GIN \cite{xu2018how}, unlike normalization layers that divide by data-dependent statistics, PolyScaleNorm is a linear scaling operation and therefore does not destroy the polynomial form.

\begin{definition}[Polynomial GIN Layer]
Let $h_v^{(l)}$ denote the representation of node $v$ at layer $l$. A PolyGIN layer first performs GIN-style aggregation,
\begin{equation}
    a_v^{(l)} = (1+\epsilon^{(l)})h_v^{(l-1)} + \sum_{u\in\mathcal{N}(v)} h_u^{(l-1)},
\end{equation}
and then applies a polynomial transformation,
\begin{equation}
    h_v^{(l)} = W_2^{(l)}\,\sigma_{\mathrm{poly}}\left(\mathrm{Norm}_{\mathrm{poly}}\left(W_1^{(l)}a_v^{(l)}\right)\right),
\end{equation}
where $W_1^{(l)}$ and $W_2^{(l)}$ are learnable linear maps.
\end{definition}

Because aggregation, linear projection, and PolyScaleNorm are linear, and PolyActivation is polynomial, each PolyGIN layer maps polynomial representations to polynomial representations.

\subsection{Degree Bound of PolyGIN}
\label{sec:degree_bound}
The exactness of APEX relies on knowing the degree of the attribution integrand. The following theorem gives the required bound.

\begin{theorem}[Degree Bound]
\label{thm:degree}
Consider a PolyGIN network with $L$ polynomial transformation blocks, each containing one PolyActivation. If the network output is denoted by $y=f^{(L)}(X)$, then $y$ is a multivariate polynomial in the input features $X$ with degree at most $2^L$. Consequently, for any input feature $x_i$, the partial derivative $\partial y/\partial x_i$ has degree at most $2^L-1$.
\end{theorem}

\begin{proof}
Let $\deg(\cdot)$ denote the maximum polynomial degree with respect to the original input features. At the input layer, $h^{(0)}=X$, so $\deg(h^{(0)})=1$. Suppose that $h^{(l-1)}$ has degree at most $2^{l-1}$. The aggregation step, linear projections, and PolyScaleNorm are all linear transformations, so they do not increase the degree. PolyActivation maps a representation $z$ to $z+\theta\odot z^2$. Therefore,
\begin{equation}
    \deg\left(z+\theta\odot z^2\right)
    \leq \max\{\deg(z),2\deg(z)\}
    \leq 2^l.
\end{equation}
By induction, the output after $L$ polynomial transformation blocks has degree at most $2^L$. Taking a first-order partial derivative reduces the degree of each nonconstant monomial by one, so $\partial y/\partial x_i$ has degree at most $2^L-1$.
\end{proof}

The bound is stated as an upper bound because degenerate parameter choices may reduce the realized degree. The upper bound is sufficient for choosing an exact quadrature rule.

\paragraph{Remark on numerical stability.}
The degree bound grows exponentially with the number of polynomial transformation blocks, and repeated quadratic transformations may lead to magnitude growth in intermediate representations. PolyGIN inserts PolyScaleNorm between the linear projection and the polynomial activation. Since PolyScaleNorm is a learnable diagonal scaling operation, it provides learnable rescaling of intermediate representations and helps mitigate magnitude growth in practice, without introducing non-polynomial terms or altering the degree bound. This mechanism does not, by itself, guarantee numerical stability for arbitrarily deep polynomial networks. Accordingly, we use a moderate number of polynomial transformation blocks in the experiments considered in this paper.

\subsection{Exact Gauss--Legendre Quadrature for Aumann--Shapley Attribution}
\label{sec:quadrature}
For a fixed feature $x_i$, define the attribution integrand
\begin{equation}
    g_i(\tau)=\frac{\partial f(\gamma(\tau))}{\partial x_i}.
\end{equation}
Since $\gamma(\tau)$ is affine in $\tau$ and $\partial f/\partial x_i$ is a polynomial of degree at most $2^L-1$ in the input features, $g_i(\tau)$ is a univariate polynomial in $\tau$ of degree at most $2^L-1$.

Gauss--Legendre quadrature states that an $m$-point Gauss–Legendre rule integrates every polynomial of degree at most $2m-1$ exactly on $[-1,1]$ \cite{golub1969calculation, lee2009robust}. Therefore, choosing
\begin{equation}
    m=2^{L-1}
\end{equation}
is sufficient for exact integration of $g_i$. Let $\{t_k,w_k\}_{k=1}^m$ be the standard Gauss--Legendre nodes and weights on $[-1,1]$. After mapping the interval $[-1,1]$ to $[0,1]$ by $\tau_k=(t_k+1)/2$, Eq.~\eqref{eq:as} becomes
\begin{equation}
\label{eq:apex_exact}
    \phi_i = \frac{x_i-x_i'}{2}\sum_{k=1}^{2^{L-1}} w_k\,
    \frac{\partial f\left(X' + \frac{t_k+1}{2}(X-X')\right)}{\partial x_i}.
\end{equation}
Under the PolyGIN degree bound, Eq.~\eqref{eq:apex_exact} evaluates the Aumann--Shapley integral exactly up to floating-point numerical precision.

\subsection{Node-level Attribution by Additive Aggregation}
\label{sec:rollup}
APEX first computes feature-level attributions. For graph interpretation, it is often useful to aggregate these feature scores into node-level scores. Let $\mathcal{F}_v$ denote the set of feature indices associated with node $v$. The signed node-level attribution is
\begin{equation}
\label{eq:node_rollup}
    \Phi_v = \sum_{i\in\mathcal{F}_v}\phi_i.
\end{equation}
This aggregation preserves the completeness of the feature-level attribution because summation over nodes is equivalent to summation over all input features:
\begin{equation}
    \sum_{v\in\mathcal{V}}\Phi_v
    = \sum_{i=1}^n \phi_i
    = f(X)-f(X').
\end{equation}

Importantly, we retain $\Phi_v$ as a signed node-level attribution rather than converting it into a nonnegative magnitude. A positive $\Phi_v$ indicates that node $v$ contributes supporting evidence to the target-class logit relative to the baseline, whereas a negative $\Phi_v$ represents opposing evidence against the target class. Under the PolyGIN degree bound, APEX eliminates quadrature truncation error; consequently, these signed node-level attributions preserve completeness up to floating-point precision and provide a numerically consistent decomposition of the prediction difference into supporting and opposing contributions. Such directional information cannot be represented by nonnegative importance scores alone.

\section{Experiments}
\label{sec:experiments}

We evaluate APEX from three perspectives. First, we test whether PolyGIN preserves predictive performance relative to a standard GIN. Second, we evaluate explanation fidelity against representative and state-of-the-art explainability baselines. Third, we measure the completeness error and runtime of the attribution computation. We evaluate APEX on five widely used graph benchmarks covering synthetic, molecular, and textual graphs. Detailed dataset statistics are provided in Appendix ~\ref{app:datasets}.

\subsection{Research Questions}
We organize the experiments around the following questions:
\begin{itemize}
    \item \textbf{RQ1: Predictive performance.} Does PolyGIN maintain competitive accuracy compared with a standard GIN?
    \item \textbf{RQ2: Attribution fidelity.} Do APEX attributions identify explanation subsets that are necessary and sufficient for preserving the model prediction, as measured by Fidelity$^+$ and Fidelity$^-$?
    \item \textbf{RQ3: Efficiency and exactness.} Does APEX reduce the completeness error and the number of required model evaluations compared with numerical integration baselines?
\end{itemize}

\subsection{Predictive Performance of PolyGIN}
\label{sec:models}
A meaningful explanation should be computed for a model with adequate predictive performance. To answer RQ1, we compare PolyGIN with a standard GIN under the same hidden dimension and batch size. All models use a hidden dimension of 300 and a batch size of 32. Dataset-specific learning rates and training epochs are reported in Table~\ref{tab:performance}. For every dataset, PolyGIN increases the number of trainable parameters by only approximately 0.12\%--0.17\% compared with GIN. Therefore, the comparison is conducted under nearly identicalmodel capacity, and the predictive performance of PolyGIN is not obtained through a substantial increase in parameter count. We use the standard train/validation/test splits for all datasets. GIN and PolyGIN use the same network depth and overall architectural configuration, differing only in the components required by the polynomial design. This matched architectural setting helps isolate the effect of the polynomial constraint on predictive performance. All models are trained using Adam with a weight decay of $5\times10^{-6}$, and the checkpoint with the best validation performance is selected for evaluation. Results are reported as the mean and standard deviation over ten independent runs with different random seeds. All experiments are conducted on a single NVIDIA GeForce RTX 4090 GPU.

\begin{table}[h]
\centering
\caption{Training configurations and predictive accuracy of GIN and PolyGIN.}
\vspace{0.5em}
\label{tab:performance}
\begin{tabular}{llcccc}
\toprule
\textbf{Dataset} & \textbf{Model} & \textbf{Learning Rate}
& \textbf{Epochs} & \textbf{Params} & \textbf{Accuracy (\%)} \\
\midrule
\multirow{2}{*}{BA-Shapes}
& GIN     & $1\times10^{-4}$ & 4000 & 546,304 & $97.29\pm0.77$ \\
& PolyGIN & $5\times10^{-5}$ & 9000 & 547,208 & $96.43\pm0.71$ \\
\midrule
\multirow{2}{*}{BBBP}
& GIN     & $1\times10^{-4}$ & 500  & 545,402 & $89.90\pm1.09$ \\
& PolyGIN & $1\times10^{-4}$ & 1000 & 546,306 & $89.85\pm0.81$ \\
\midrule
\multirow{2}{*}{BACE}
& GIN     & $5\times10^{-5}$ & 1000 & 545,402 & $83.62\pm1.19$ \\
& PolyGIN & $5\times10^{-5}$ & 1000 & 546,306 & $84.34\pm1.86$ \\
\midrule
\multirow{2}{*}{Graph-SST2}   
& GIN     & $1\times10^{-3}$ & 20 & 773,102 & $90.95\pm0.20$ \\
& PolyGIN & $1\times10^{-3}$ & 20 & 774,006 & $90.23\pm0.17$ \\
\midrule
\multirow{2}{*}{Mutagenicity}
& GIN     & $1\times10^{-4}$ & 100 & 546,902 & $81.78\pm0.40$ \\
& PolyGIN & $1\times10^{-3}$ & 800 & 547,806 & $80.98\pm0.61$ \\
\bottomrule
\end{tabular}
\end{table}

The results show that PolyGIN achieves accuracy close to the standard GIN on all five datasets. On BACE, PolyGIN obtains a higher accuracy than GIN; on BBBP, Graph-SST2, BA-Shapes, and Mutagenicity, the decrease is below one percentage point. These results suggest that the polynomial constraint does not substantially degrade predictive performance under the evaluated experimental conditions, providing a reasonable basis for the attribution experiments.

\subsection{Attribution Fidelity}
\label{sec:fidelity}
To answer RQ2, we evaluate whether the selected key factors are necessary and sufficient for the model prediction. We compare APEX with representative and widely used GNN explanation baselines, including GNNExplainer\cite{ying2019gnnexplainer}, PGExplainer\cite{luo2020parameterized}, GradCAM\cite{pope2019explainability}, FlowX\cite{gui2023flowx}, and Integrated Gradients \cite{sundararajan2017axiomatic}. We evaluate all methods on the same correctly classified test graphs and explain the model-predicted class. APEX attributes the corresponding pre-softmax logit from a zero-feature baseline and ranks nodes in descending order of their signed attribution scores. For the fidelity experiments, IG and its numerical variants use 50 integration points. Edge-based explanations are converted to node masks using the endpoints of the highest-scoring edges. Masked node features are replaced by random within-graph feature permutations while preserving the graph topology, and the resulting prediction probabilities are averaged over ten permutations. We follow the necessity–sufficiency evaluation protocol formalized in GraphFramEx \cite{amara2022graphframex} and evaluate all methods at matched sparsity levels. For an input graph $G$, let $c$ be the class predicted by the original model and let $p_c(G)$ be the predicted probability of class $c$. Given an explanation subset $S$, let $G\setminus S$ denote the graph after removing or masking the selected key factors, and let $G[S]$ denote the graph retaining only the selected key factors. We use
\begin{equation}
    \mathrm{Fidelity}^{+} = \mathbb{E}_{G}\left[p_c(G)-p_c(G\setminus S)\right]
\end{equation}
for necessary explanations, where a larger value indicates that the removed components were important to the prediction. We use
\begin{equation}
    \mathrm{Fidelity}^{-} = \mathbb{E}_{G}\left[p_c(G)-p_c(G[S])\right]
\end{equation}
for sufficient explanations, where a smaller value indicates that the retained components preserve the original prediction more effectively. As with removal-based fidelity metrics generally, these scores may be affected by distribution shift induced by graph perturbation; we therefore use them as comparative model-faithfulness measures under a shared perturbation protocol rather than as direct measures of semantic correctness \cite{zheng2024towards}.

\begin{figure*}[h]
    \centering
    \IfFileExists{figure/F+.pdf}{%
        \includegraphics[width=1.0\linewidth]{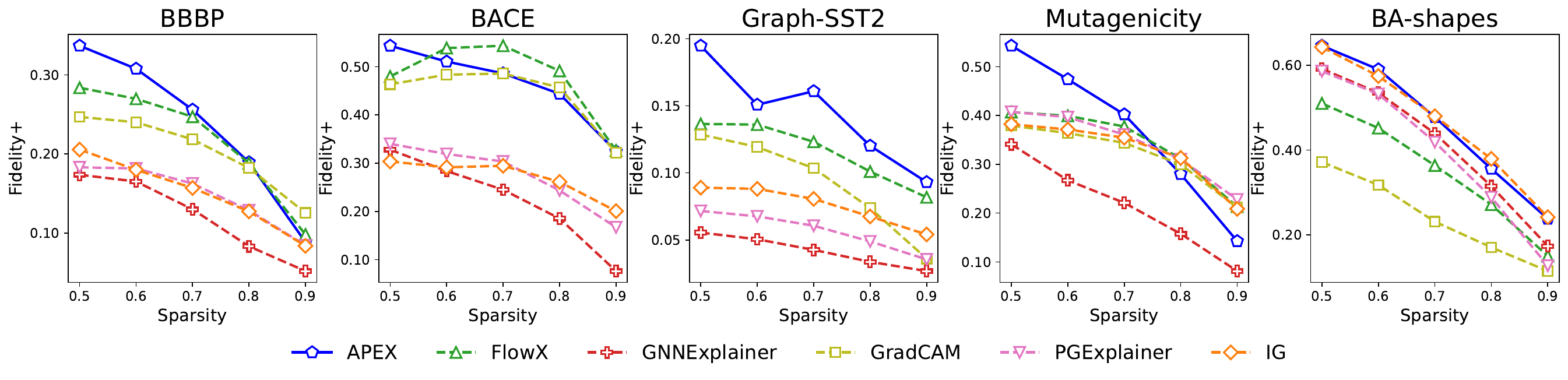}%
    }{%
        \fbox{\parbox{0.92\linewidth}{\centering Placeholder for Fidelity$^+$ results across datasets and sparsity levels.}}%
    }
    \caption{Necessary explanation comparison measured by Fidelity$^+$. Higher values indicate that the selected key factors are more necessary for preserving the original prediction.}
    \label{fig:fidelity_plus}
\end{figure*}

\paragraph{Necessary explanations.}
Figure~\ref{fig:fidelity_plus} compares Fidelity$^+$ across datasets and sparsity levels. Under these experimental conditions, APEX obtains higher Fidelity$^+$ than the compared state-of-the-art baselines across the evaluated datasets. This indicates that the key factors selected by APEX have a stronger effect on the model output when removed. These results demonstrate the framework-level effectiveness of the proposed architecture–attribution co-design under the shared evaluation protocol.

\begin{figure*}[h]
    \centering
    \IfFileExists{figure/F-.pdf}{%
        \includegraphics[width=1.0\linewidth]{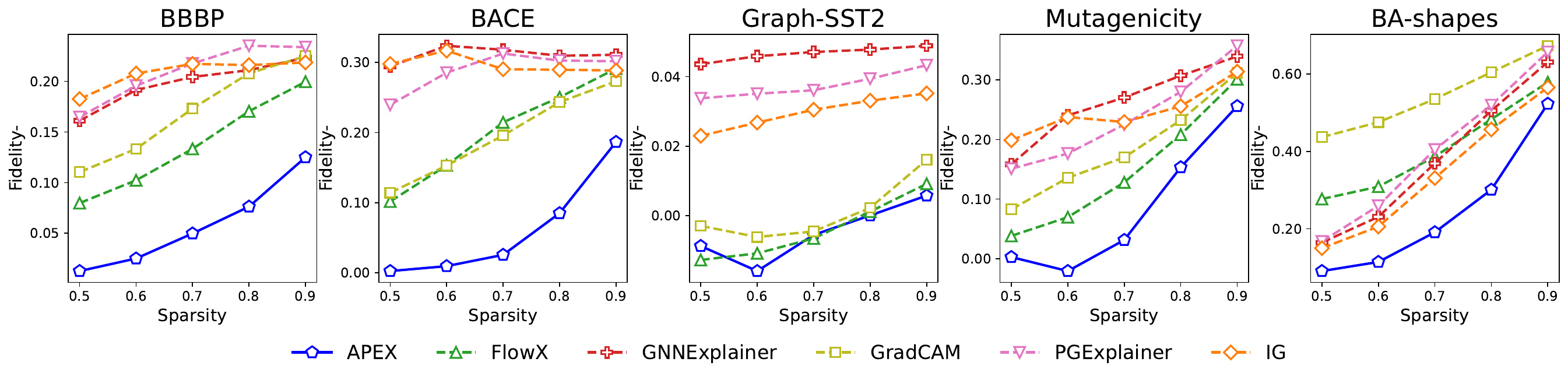}%
    }{%
        \fbox{\parbox{0.92\linewidth}{\centering Placeholder for Fidelity$^-$ results across datasets and sparsity levels.}}%
    }
    \caption{Sufficient explanation comparison measured by Fidelity$^-$. Lower values indicate that the retained key factors better preserve the original prediction.}
    \label{fig:fidelity_minus}
\end{figure*}

\paragraph{Sufficient explanations.}
Figure~\ref{fig:fidelity_minus} reports Fidelity$^-$. APEX achieves lower Fidelity$^-$ values than the compared state-of-the-art baselines under the evaluated experimental conditions, meaning that the selected key factors preserve more of the original prediction when kept alone. These results show that, under the evaluated conditions, the complete APEX framework achieves stronger sufficiency fidelity than the compared baselines.

\subsection{Efficiency and Completeness Error}
\label{sec:efficiency}
To answer RQ3, we evaluate both the computational cost and the attribution completeness error. For path-integral attribution,exact straight-line path attribution satisfies completeness by the fundamental theorem of calculus \cite{sundararajan2017axiomatic}
\begin{equation}
    \sum_{i=1}^{n}\phi_i = f(X)-f(X').
\end{equation}
For an approximate method, we measure the completeness error as
\begin{equation}
\label{eq:completeness_error}
    \Delta_{\mathrm{comp}} = \left|\sum_{i=1}^{n}\hat{\phi}_i - \big(f(X)-f(X')\big)\right|.
\end{equation}
Exact continuous Aumann–Shapley attribution rigorously satisfies completeness in theory: the sum of all feature attributions equals the change in the model output relative to the baseline, with no error introduced by the attribution allocation itself. Conventional integrated gradients methods numerically approximate the continuous path integral using a finite number of sampling points; consequently, their nonzero completeness error fundamentally arises from the quadrature error induced by path discretization. Appendix~\ref{app:truncation_error} further proves that, for integrated gradients approximated using Riemann sums, the completeness error is exactly equal to the corresponding numerical quadrature error. In contrast, APEX completely eliminates this quadrature error under exact arithmetic, thereby exactly recovering the completeness property of continuous Aumann–Shapley attribution; in practical computation, only rounding errors due to finite-precision floating-point arithmetic remain.

\begin{figure*}[t]
    \centering
    \IfFileExists{figure/efficiency_gap.pdf}{%
        \includegraphics[width=1.0\linewidth]{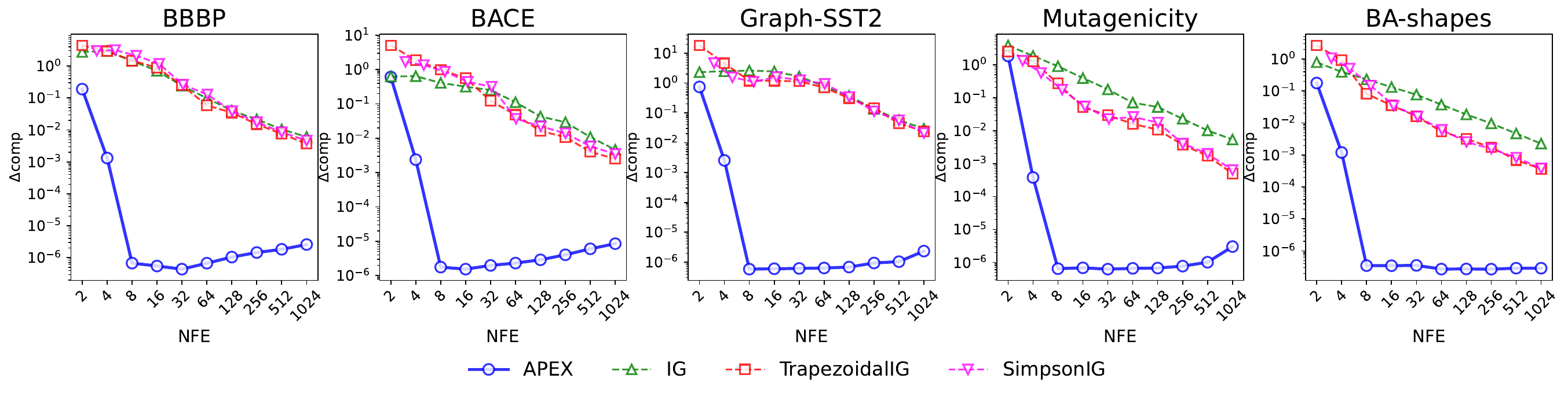}%
    }{%
        \fbox{\parbox{0.92\linewidth}{\centering Placeholder for completeness error / efficiency gap curves with respect to the number of forward evaluations.}}%
    }
    \caption{Completeness error with respect to the number of forward evaluations (NFE). Numerical integration baselines reduce the error gradually as more evaluation points are used. APEX reaches numerical precision at the number of points predicted by the PolyGIN degree bound.}
    \label{fig:efficiency_gap}
\end{figure*}

Figure~\ref{fig:efficiency_gap} compares APEX with numerical integration baselines, including standard IG, trapezoidal IG, and Simpson IG\cite{suli2003introduction}. The numerical baselines reduce completeness error as the number of evaluation points increases, but the error remains visible when the number of evaluations is limited. In contrast, for the 4-block PolyGIN used in our experiments, Theorem~\ref{thm:degree} gives an integrand degree bound of $2^4-1=15$, so $2^{4-1}=8$ Gauss--Legendre points are sufficient. With only eight evaluation points, the completeness error drops directly to the numerical precision limit of FP32, in exact agreement with the theoretical prediction. This result confirms that, under the proposed PolyGIN architecture and logit-based attribution setting, APEX does not merely reduce numerical integration error; rather, it theoretically eliminates the quadrature error induced by finite path sampling. The residual deviations are at the scale expected from finite-precision FP32 computation and are therefore consistent with floating-point roundoff rather than quadrature truncation error, rather than from approximation error in the attribution integral. Consequently, once the model depth is fixed, APEX avoids the sampling-resolution sweep required by numerical IG and achieves constant-level integration complexity with respect to $N_{\mathrm{step}}$.

\section{Conclusion}
\label{sec:conclusions}
This paper introduced APEX, a framework for exact Aumann--Shapley attribution in polynomial graph neural networks. The central idea is to co-design the model architecture and attribution algorithm: PolyGIN constrains the GNN mapping to a bounded polynomial space, and Gauss--Legendre quadrature uses the resulting degree bound to evaluate the path integral exactly with a fixed number of deterministic points. The feature-level attributions can be aggregated into node-level explanations while preserving completeness. Experiments on five graph benchmarks show that The complete APEX framework achieves higher attribution fidelity than the compared baselines, while its architecture–attribution co-design enables exact path-integral evaluation, up to floating-point precision, with a fixed and deterministic computational budget.

The exactness guarantee of APEX is conditional on the polynomial architecture and on computing attributions for scalar polynomial scores, such as logits before non-polynomial output transformations. It does not directly apply to arbitrary GNNs with ReLU activations, softmax probabilities, batch normalization, attention softmax, or other non-polynomial components. In addition, exact quadrature guarantees mathematical completeness of the path integral under the model, but it does not by itself guarantee agreement with human-interpretable ground truth in every application. Extending comparable guarantees to broader graph architectures and studying the semantic validity of the resulting explanations are important directions for future work.

\bibliographystyle{unsrt}
\bibliography{references}

@inproceedings{fan2019graph,
  title={Graph neural networks for social recommendation},
  author={Fan, Wenqi and Ma, Yao and Li, Qing and He, Yuan and Zhao, Eric and Tang, Jiliang and Yin, Dawei},
  booktitle={The world wide web conference},
  pages={417--426},
  year={2019}
}

@article{wu2022graph,
  title={Graph neural networks in recommender systems: a survey},
  author={Wu, Shiwen and Sun, Fei and Zhang, Wentao and Xie, Xu and Cui, Bin},
  journal={ACM Computing Surveys},
  volume={55},
  number={5},
  pages={1--37},
  year={2022},
  publisher={ACM New York, NY}
}

@article{jiang2021could,
  title={Could graph neural networks learn better molecular representation for drug discovery? A comparison study of descriptor-based and graph-based models},
  author={Jiang, Dejun and Wu, Zhenxing and Hsieh, Chang-Yu and Chen, Guangyong and Liao, Ben and Wang, Zhe and Shen, Chao and Cao, Dongsheng and Wu, Jian and Hou, Tingjun},
  journal={Journal of cheminformatics},
  volume={13},
  number={1},
  pages={12},
  year={2021},
  publisher={Springer}
}

@inproceedings{xu2018how,
  title={How Powerful are Graph Neural Networks?},
  author={Keyulu Xu and Weihua Hu and Jure Leskovec and Stefanie Jegelka},
  booktitle={International Conference on Learning Representations},
  year={2019},
  url={https://openreview.net/forum?id=ryGs6iA5Km},
}

@article{ying2019gnnexplainer,
  title={Gnnexplainer: Generating explanations for graph neural networks},
  author={Ying, Zhitao and Bourgeois, Dylan and You, Jiaxuan and Zitnik, Marinka and Leskovec, Jure},
  journal={Advances in neural information processing systems},
  volume={32},
  year={2019}
}

@article{luo2020parameterized,
  title={Parameterized explainer for graph neural network},
  author={Luo, Dongsheng and Cheng, Wei and Xu, Dongkuan and Yu, Wenchao and Zong, Bo and Chen, Haifeng and Zhang, Xiang},
  journal={Advances in neural information processing systems},
  volume={33},
  pages={19620--19631},
  year={2020}
}

@inproceedings{yuan2021explainability,
  title={On explainability of graph neural networks via subgraph explorations},
  author={Yuan, Hao and Yu, Haiyang and Wang, Jie and Li, Kang and Ji, Shuiwang},
  booktitle={International conference on machine learning},
  pages={12241--12252},
  year={2021},
  organization={PMLR}
}

@article{wu2018moleculenet,
  title={MoleculeNet: a benchmark for molecular machine learning},
  author={Wu, Zhenqin and Ramsundar, Bharath and Feinberg, Evan N and Gomes, Joseph and Geniesse, Caleb and Pappu, Aneesh S and Leswing, Karl and Pande, Vijay},
  journal={Chemical science},
  volume={9},
  number={2},
  pages={513--530},
  year={2018},
  publisher={Royal Society of Chemistry}
}

@article{yuan2022explainability,
  title={Explainability in graph neural networks: A taxonomic survey},
  author={Yuan, Hao and Yu, Haiyang and Gui, Shurui and Ji, Shuiwang},
  journal={IEEE transactions on pattern analysis and machine intelligence},
  volume={45},
  number={5},
  pages={5782--5799},
  year={2022},
  publisher={IEEE}
}

@inproceedings{pope2019explainability,
  title={Explainability methods for graph convolutional neural networks},
  author={Pope, Phillip E and Kolouri, Soheil and Rostami, Mohammad and Martin, Charles E and Hoffmann, Heiko},
  booktitle={Proceedings of the IEEE/CVF conference on computer vision and pattern recognition},
  pages={10772--10781},
  year={2019}
}

@inproceedings{sundararajan2017axiomatic,
  title={Axiomatic attribution for deep networks},
  author={Sundararajan, Mukund and Taly, Ankur and Yan, Qiqi},
  booktitle={International conference on machine learning},
  pages={3319--3328},
  year={2017},
  organization={PMLR}
}

@article{gui2023flowx,
  title={Flowx: Towards explainable graph neural networks via message flows},
  author={Gui, Shurui and Yuan, Hao and Wang, Jie and Lao, Qicheng and Li, Kang and Ji, Shuiwang},
  journal={IEEE Transactions on Pattern Analysis and Machine Intelligence},
  volume={46},
  number={7},
  pages={4567--4578},
  year={2023},
  publisher={IEEE}
}

@article{shapley1953value,
  title={A value for n-person games},
  author={Shapley, Lloyd S and others},
  year={1953},
  publisher={Princeton University Press Princeton}
}

@book{aumann2015values,
  title={Values of non-atomic games},
  author={Aumann, Robert J and Shapley, Lloyd S},
  year={2015},
  publisher={Princeton University Press}
}

@article{hesse2021fast,
  title={Fast axiomatic attribution for neural networks},
  author={Hesse, Robin and Schaub-Meyer, Simone and Roth, Stefan},
  journal={Advances in Neural Information Processing Systems},
  volume={34},
  pages={19513--19524},
  year={2021}
}

@article{kileel2019expressive,
  title={On the expressive power of deep polynomial neural networks},
  author={Kileel, Joe and Trager, Matthew and Bruna, Joan},
  journal={Advances in neural information processing systems},
  volume={32},
  year={2019}
}

@inproceedings{chrysos2020p,
  title={P-nets: Deep polynomial neural networks},
  author={Chrysos, Grigorios G and Moschoglou, Stylianos and Bouritsas, Giorgos and Panagakis, Yannis and Deng, Jiankang and Zafeiriou, Stefanos},
  booktitle={Proceedings of the IEEE/CVF Conference on Computer Vision and Pattern Recognition},
  pages={7325--7335},
  year={2020}
}

@article{defferrard2016convolutional,
  title={Convolutional neural networks on graphs with fast localized spectral filtering},
  author={Defferrard, Micha{\"e}l and Bresson, Xavier and Vandergheynst, Pierre},
  journal={Advances in neural information processing systems},
  volume={29},
  year={2016}
}

@inproceedings{gilmer2017neural,
  title={Neural message passing for quantum chemistry},
  author={Gilmer, Justin and Schoenholz, Samuel S and Riley, Patrick F and Vinyals, Oriol and Dahl, George E},
  booktitle={International conference on machine learning},
  pages={1263--1272},
  year={2017},
  organization={Pmlr}
}

@inproceedings{yao2019graph,
  title={Graph convolutional networks for text classification},
  author={Yao, Liang and Mao, Chengsheng and Luo, Yuan},
  booktitle={Proceedings of the AAAI conference on artificial intelligence},
  volume={33},
  number={01},
  pages={7370--7377},
  year={2019}
}

@inproceedings{sundararajan2020many,
  title={The many Shapley values for model explanation},
  author={Sundararajan, Mukund and Najmi, Amir},
  booktitle={International conference on machine learning},
  pages={9269--9278},
  year={2020},
  organization={PMLR}
}

@article{golub1969calculation,
  title={Calculation of Gauss quadrature rules},
  author={Golub, Gene H and Welsch, John H},
  journal={Mathematics of computation},
  volume={23},
  number={106},
  pages={221--230},
  year={1969}
}

@article{lee2009robust,
  title={Robust design with arbitrary distributions using Gauss-type quadrature formula},
  author={Lee, Sang Hoon and Chen, Wei and Kwak, Byung Man},
  journal={Structural and Multidisciplinary Optimization},
  volume={39},
  number={3},
  pages={227--243},
  year={2009},
  publisher={Springer}
}

@article{amara2022graphframex,
  title={Graphframex: Towards systematic evaluation of explainability methods for graph neural networks},
  author={Amara, Kenza and Ying, Rex and Zhang, Zitao and Han, Zhihao and Shan, Yinan and Brandes, Ulrik and Schemm, Sebastian and Zhang, Ce},
  journal={arXiv preprint arXiv:2206.09677},
  year={2022}
}

@inproceedings{zheng2024towards,
  title={Towards robust fidelity for evaluating explainability of graph neural networks},
  author={Zheng, Xu and Shirani, Farhad and Wang, Tianchun and Cheng, Wei and Chen, Zhuomin and Chen, Haifeng and Wei, Hua and Luo, Dongsheng},
  booktitle={International Conference on Learning Representations},
  volume={2024},
  pages={12250--12275},
  year={2024}
}

@article{kazius2005derivation,
  title={Derivation and validation of toxicophores for mutagenicity prediction},
  author={Kazius, Jeroen and McGuire, Ross and Bursi, Roberta},
  journal={Journal of medicinal chemistry},
  volume={48},
  number={1},
  pages={312--320},
  year={2005},
  publisher={ACS Publications}
}

@inproceedings{morris2019weisfeiler,
  title={Weisfeiler and leman go neural: Higher-order graph neural networks},
  author={Morris, Christopher and Ritzert, Martin and Fey, Matthias and Hamilton, William L and Lenssen, Jan Eric and Rattan, Gaurav and Grohe, Martin},
  booktitle={Proceedings of the AAAI conference on artificial intelligence},
  volume={33},
  number={01},
  pages={4602--4609},
  year={2019}
}

@book{suli2003introduction,
  title={An introduction to numerical analysis},
  author={S{\"u}li, Endre and Mayers, David F},
  year={2003},
  publisher={Cambridge university press}
}

\clearpage

\appendix

\section{Dataset Statistics}
\label{app:datasets}

We use BA-Shapes from GNNExplainer \cite{ying2019gnnexplainer}; BBBP and BACE as curated in MoleculeNet \cite{wu2018moleculenet}; Graph-SST2 from the structured NLP graph-explanation benchmark \cite{yuan2022explainability}; and the Mutagenicity dataset derived from the Ames mutagenicity collection of Kazius et al. \cite{kazius2005derivation}. We report the statistics of the five evaluated datasets in Table~\ref{tab:datasets}. The benchmarks cover synthetic node classification, molecular graph classification, and text graph classification, thereby evaluating APEX across structurally distinct graph domains.

\begin{table}[h]
\centering
\caption{Statistics of the evaluated datasets. NC and GC denote node and graph classification, respectively.}
\label{tab:datasets}
\vspace{0.5em}
\small
\setlength{\tabcolsep}{8pt}
\renewcommand{\arraystretch}{1.08}
\begin{tabular}{lcccc}
\toprule
\textbf{Dataset} & \textbf{Task} & \textbf{\# Graphs} & \textbf{Avg. \# Nodes} & \textbf{\# Classes} \\
\midrule
BA-Shapes    & NC & 1      & 700  & 4 \\
BBBP         & GC & 2,039  & 24.1 & 2 \\
BACE         & GC & 1,513  & 34.1 & 2 \\
Graph-SST2   & GC & 70,042 & 10.2 & 2 \\
Mutagenicity & GC & 4,337  & 30.3 & 2 \\
\bottomrule
\end{tabular}
\vspace{-0.8em}
\end{table}

\section{How Powerful is PolyGIN?}
\label{app:polyproof}

In this appendix, we analyze the expressive power of PolyGIN under the standard Weisfeiler--Lehman (WL) framework for message-passing graph neural networks. We follow the WL-based expressivity framework developed for message-passing GNNs and GIN \cite{morris2019weisfeiler, xu2018how}. Our goal is to clarify the representational role of the polynomial constraint introduced in PolyGIN. The analysis follows the same line of reasoning used to characterize the expressive power of Graph Isomorphism Networks (GIN): a message-passing GNN can be as discriminative as the 1-dimensional WL test if its neighborhood aggregation, node update, and graph-level readout functions are injective over the corresponding multisets.

We show that PolyGIN has the same upper bound as standard message-passing GNNs: it cannot distinguish graph pairs that 1-WL fails to distinguish. We then show that, under an explicit injectivity condition on the polynomial transformation and graph-level readout, PolyGIN can match the discriminative power of 1-WL. This conditional statement is important: the polynomial degree guarantee used by APEX for exact Aumann--Shapley quadrature does not by itself imply WL-level expressive power. Rather, the exactness guarantee and the WL expressivity guarantee concern two different properties of the architecture.

\subsection{Preliminaries}
Let $G=(V,E)$ be a graph with node features $x_v$ for $v\in V$. We consider a message-passing architecture in which each node representation is updated by aggregating the representations of its neighbors. PolyGIN follows the GIN aggregation rule but replaces the usual nonlinear transformation module with polynomial and linear operations. Specifically, the $\ell$-th PolyGIN layer computes:
\begin{equation}
a_v^{(\ell)}=(1+\epsilon^{(\ell)})h_v^{(\ell-1)}+\sum_{u\in\mathcal{N}(v)}h_u^{(\ell-1)},
\end{equation}
followed by $h_v^{(\ell)}=P_\ell(a_v^{(\ell)})$, where:
\begin{equation}
P_\ell(z)=W_2^{(\ell)}\sigma_{\mathrm{poly}}\left(\operatorname{Norm}_{\mathrm{poly}}\left(W_1^{(\ell)}z\right)\right),
\end{equation}

\begin{equation}
    \sigma_{\mathrm{poly}}(z)=z+\theta^{(\ell)}\odot z^2,
\qquad
\operatorname{Norm}_{\mathrm{poly}}(z)=s^{(\ell)}\odot z.
\end{equation}

Here $W_1^{(\ell)}$ and $W_2^{(\ell)}$ are learnable linear maps, $s^{(\ell)}$ is a learnable scaling vector, and $\theta^{(\ell)}$ controls the quadratic polynomial activation.

We assume throughout this appendix that the input feature universe is countable and that node degrees are bounded. This is the same setting commonly used in WL-style expressivity analyses of message-passing GNNs\cite{xu2018how,morris2019weisfeiler}. For a finite graph dataset, this assumption is automatically satisfied when node features are categorical, discrete labels, or finite-precision vectors.

We also use the notion of a multiset. A multiset is a set-like object that allows repeated elements. For a node v, the collection $\left\{ h_u^{(\ell-1)} : u\in\mathcal{N}(v) \right\}$ is naturally a multiset, since different neighbors may have identical representations.

The 1-dimensional WL test iteratively refines node colors by hashing the pair consisting of a node’s current color and the multiset of its neighbors’ current colors. Thus, at iteration $\ell$, WL assigns a new color according to:
\begin{equation}
    c_v^{(\ell)}=\operatorname{HASH}\left(
c_v^{(\ell-1)},
\left\{c_u^{(\ell-1)}:u\in\mathcal{N}(v)\right\}
\right),
\end{equation}

where $\operatorname{HASH}$ is injective on all distinct pairs. A message-passing GNN can match the WL refinement process only if its update rule preserves the distinctness of these pairs.

\subsection{PolyGIN is Upper Bounded by 1-WL}
We first establish that PolyGIN cannot be more powerful than the 1-WL test. This upper bound is not specific to the polynomial transformation; it follows from the fact that PolyGIN is a permutation-invariant neighborhood aggregation architecture.

\subsubsection{Proposition 1.}
Let $G_1$ and $G_2$ be two graphs with initial node features. If the 1-WL test assigns identical color multisets to $G_1$ and $G_2$ at every iteration, then any PolyGIN with the same initial node features and any fixed parameters assigns identical graph-level representations to $G_1$ and $G_2$, provided the graph-level readout is permutation invariant. Consequently, PolyGIN is at most as powerful as the 1-WL test in distinguishing non-isomorphic graphs.

\begin{proof}
We prove by induction on the layer index $\ell$ that if two nodes have the same WL color after $\ell$ WL iterations, then PolyGIN assigns them the same representation after $\ell$ layers.

At $\ell$=0, the claim follows from the initialization: nodes with the same initial WL color have the same initial feature representation.

Assume the claim holds for layer $\ell$-1. Consider two nodes $v$ and $v'$ whose WL colors at iteration $\ell$ are the same. By the WL update rule, this means that their previous colors are the same and that their neighbor-color multisets are the same:
\begin{equation}
    c_v^{(\ell-1)}=c_{v'}^{(\ell-1)}
\end{equation}
and
\begin{equation}
    \left\{c_u^{(\ell-1)}:u\in\mathcal{N}(v)\right\}
=
\left\{c_{u'}^{(\ell-1)}:u'\in\mathcal{N}(v')\right\}.
\end{equation}
By the induction hypothesis, nodes with the same WL color at iteration $\ell$-1 have the same PolyGIN representation. Therefore,
\begin{equation}
    h_v^{(\ell-1)}=h_{v'}^{(\ell-1)}
\end{equation}
and
\begin{equation}
    \left\{h_u^{(\ell-1)}:u\in\mathcal{N}(v)\right\}
=
\left\{h_{u'}^{(\ell-1)}:u'\in\mathcal{N}(v')\right\}
\end{equation}
as multisets.

Since PolyGIN uses a permutation-invariant sum aggregation, the aggregated vectors are equal:
\begin{equation}
    a_v^{(\ell)}
=
(1+\epsilon^{(\ell)})h_v^{(\ell-1)}
+
\sum_{u\in\mathcal{N}(v)}h_u^{(\ell-1)}
=
(1+\epsilon^{(\ell)})h_{v'}^{(\ell-1)}
+
\sum_{u'\in\mathcal{N}(v')}h_{u'}^{(\ell-1)}
=
a_{v'}^{(\ell)}.
\end{equation}
Applying the same polynomial transformation $P_\ell$ to both sides gives
\begin{equation}
h_v^{(\ell)}=P_\ell(a_v^{(\ell)})=P_\ell(a_{v'}^{(\ell)})=h_{v'}^{(\ell)}.
\end{equation}
Thus, the claim holds for layer $\ell$.

If 1-WL assigns the same color multiset to $G_1$ and $G_2$ at every iteration, then after every PolyGIN layer the two graphs also have the same multiset of node representations. Any permutation\-invariant graph\-level readout must therefore produce the same graph representation for both graphs. Hence PolyGIN cannot distinguish any pair of graphs that 1-WL cannot distinguish.
\end{proof}

This proposition shows that the polynomial design does not move PolyGIN beyond the standard 1-WL expressivity barrier of message-passing GNNs. In particular, PolyGIN cannot distinguish classical 1-WL-indistinguishable graph pairs such as certain regular graph pairs without additional structural information or higher-order mechanisms.

\subsection{Injectivity Condition for Matching 1-WL}

We now give a sufficient condition under which PolyGIN matches the discriminative power of 1-WL. The central issue is whether the PolyGIN update maps distinct WL refinement pairs to distinct node representations.

For each layer $\ell$, define the WL refinement pair associated with node v as $\mathcal{P}_v^{(\ell)}=\left(h_v^{(\ell-1)},\left\{h_u^{(\ell-1)}:u\in\mathcal{N}(v)\right\}
\right).$
PolyGIN first maps this pair to
\begin{equation}
    S_\ell\left(\mathcal{P}_v^{(\ell)}\right)
=
(1+\epsilon^{(\ell)})h_v^{(\ell-1)}
+
\sum_{u\in\mathcal{N}(v)}h_u^{(\ell-1)},
\end{equation}
and then applies $P_\ell$. Therefore, the effective update function is $U_\ell=P_\ell\circ S_\ell.$

We now state the layer-wise injectivity assumption.

\subsubsection{Assumption 1.}
For every layer $\ell$, the effective PolyGIN update
\begin{equation}
U_\ell:
\left(
h_v^{(\ell-1)},
\left\{h_u^{(\ell-1)}:u\in\mathcal{N}(v)\right\}
\right)
\mapsto
h_v^{(\ell)}
\end{equation}
is injective over all node-centered rooted neighborhoods that can occur in the considered graph family.
Equivalently, for any two distinct pairs $(c,X)\neq(c',X')$, where $c,c'$ are center-node representations and X,X' are bounded-size multisets of neighbor representations, we require
\begin{equation}
    P_\ell\left((1+\epsilon^{(\ell)})c+\sum_{x\in X}x\right)
\neq
P_\ell\left((1+\epsilon^{(\ell)})c'+\sum_{x'\in X'}x'\right).
\end{equation}
This assumption separates two possible sources of collision. First, the GIN-style sum
$(c,X)\mapsto (1+\epsilon)c+\sum_{x\in X}x$ must not collapse distinct center-neighborhood pairs. Second, the polynomial transformation $P_\ell$ must not collapse distinct aggregated vectors that encode different WL refinement pairs.

The first part is the same multiset-injectivity issue addressed in the GIN analysis. For countable feature spaces and bounded multisets, the GIN analysis establishes the existence of injective sum-based multiset encodings \cite{xu2018how}. The second part is specific to PolyGIN: the polynomial block must preserve the distinctness of the relevant aggregated values.

\subsubsection{Proposition 2.}
Let $\mathcal{A}_\ell$ denote the set of aggregated representations that can occur at layer $\ell$. Consider the polynomial transformation $P_\ell$ defined in Eqs.~(18)--(19). Suppose that $W_1^{(\ell)}$ is injective on $\mathcal{A}_\ell$, every entry of $s^{(\ell)}$ is nonzero, and there exists a constant $B_\ell>0$ such that every coordinate of
\begin{equation}
\operatorname{Norm}_{\mathrm{poly}}
\left(
W_1^{(\ell)}a
\right),
\qquad a\in\mathcal{A}_\ell,
\end{equation}
lies in $[-B_\ell,B_\ell]$. If
\begin{equation}
\left|\theta_j^{(\ell)}\right|
<
\frac{1}{2B_\ell}
\end{equation}
for every coordinate $j$, and $W_2^{(\ell)}$ is injective on the representations produced by the polynomial activation, then $P_\ell$ is injective on $\mathcal{A}_\ell$.

\textit{Proof.}
Since $W_1^{(\ell)}$ is injective on $\mathcal{A}_\ell$ and every entry of $s^{(\ell)}$ is nonzero, the linear projection and PolyScaleNorm preserve the distinctness of the aggregated representations. For the $j$-th coordinate, the polynomial activation in Eq.~(19) is the scalar function
\begin{equation}
q_j(t)=t+\theta_j^{(\ell)}t^2.
\end{equation}
For every $t\in[-B_\ell,B_\ell]$, its derivative satisfies
\begin{equation}
q_j'(t)
=
1+2\theta_j^{(\ell)}t
\geq
1-2\left|\theta_j^{(\ell)}\right|B_\ell
>
0.
\end{equation}
Therefore, $q_j$ is strictly increasing and hence injective on $[-B_\ell,B_\ell]$. Since the polynomial activation is applied element-wise, it preserves the distinctness of all representations in the relevant domain. Finally, because $W_2^{(\ell)}$ is injective on the resulting representations, the complete transformation $P_\ell$ is injective on $\mathcal{A}_\ell$.

For a finite graph family or dataset, $\mathcal{A}_\ell$ is finite and therefore bounded, so such a constant $B_\ell$ exists. Moreover, the condition permits nonzero values of $\theta^{(\ell)}$, and thus injectivity does not require removing the quadratic term.

Proposition 2 gives explicit sufficient conditions under which the polynomial transformation $P_\ell$ preserves distinct aggregated representations. Therefore, if the GIN-style aggregation is injective and the conditions of Proposition 2 hold, then the effective update $U_\ell=P_\ell\circ S_\ell$ satisfies the layer-wise injectivity condition in Assumption 1.

\subsection{PolyGIN Can Match 1-WL Under Injective Updates}
\begin{theorem}
    Consider an L-layer PolyGIN. Assume that the input feature universe is countable, neighborhood sizes are bounded, Assumption 1 holds for every layer, and the graph-level readout is injective over multisets of node representations. Then PolyGIN is as powerful as the 1-dimensional Weisfeiler–Lehman test: for any two graphs $G_1$ and $G_2$, if 1-WL distinguishes $G_1$ and $G_2$ within $L$ iterations, then PolyGIN maps $G_1$ and $G_2$ to different graph-level representations.
\end{theorem}

\begin{proof}
    We prove that, for every layer $\ell$, the PolyGIN representation $h_v^{(\ell)}$ uniquely determines the WL color $c_v^{(\ell)}$, and distinct WL colors are assigned distinct PolyGIN representations.

    At $\ell=0$, the claim follows from the initialization. Nodes with different initial features have different initial WL colors. Under a one-hot or otherwise injective initial feature encoding, different initial WL colors correspond to different initial PolyGIN representations.
    Assume the claim holds for layer $\ell-1$. Consider two nodes $v$ and $v'$. If their WL colors at iteration $\ell$ are different, then by the WL refinement rule,
    \begin{equation}
        \left(
c_v^{(\ell-1)},
\left\{c_u^{(\ell-1)}:u\in\mathcal{N}(v)\right\}
\right)
\neq
\left(
c_{v'}^{(\ell-1)},
\left\{c_{u'}^{(\ell-1)}:u'\in\mathcal{N}(v')\right\}
\right).
    \end{equation}
    By the induction hypothesis, different WL colors at iteration $\ell-1$ correspond to different PolyGIN representations. Therefore, the two WL refinement pairs induce two distinct PolyGIN representation pairs:
    \begin{equation}
        \left(
h_v^{(\ell-1)},
\left\{h_u^{(\ell-1)}:u\in\mathcal{N}(v)\right\}
\right)
\neq
\left(
h_{v'}^{(\ell-1)},
\left\{h_{u'}^{(\ell-1)}:u'\in\mathcal{N}(v')\right\}
\right).
    \end{equation}
    By Assumption 1, the effective PolyGIN update $U_\ell$ is injective over such pairs. Hence,
    \begin{equation}
        h_v^{(\ell)}
=
U_\ell
\left(
h_v^{(\ell-1)},
\left\{h_u^{(\ell-1)}:u\in\mathcal{N}(v)\right\}
\right)
\neq
U_\ell
\left(
h_{v'}^{(\ell-1)},
\left\{h_{u'}^{(\ell-1)}:u'\in\mathcal{N}(v')\right\}
\right)
=
h_{v'}^{(\ell)}.
    \end{equation}
    Thus, distinct WL colors at iteration $\ell$ are mapped to distinct PolyGIN representations. This proves the induction step.
    Now suppose 1-WL distinguishes $G_1$ and $G_2$ within L iterations. Then for some $\ell\le L$, the multisets of WL colors over the nodes of $G_1$ and $G_2$ differ:
    \begin{equation}
        \left\{c_v^{(\ell)}:v\in V(G_1)\right\}
\neq
\left\{c_v^{(\ell)}:v\in V(G_2)\right\}.
    \end{equation}
    By the result above, the corresponding multisets of PolyGIN node representations also differ:
    \begin{equation}
        \left\{h_v^{(\ell)}:v\in V(G_1)\right\}
\neq
\left\{h_v^{(\ell)}:v\in V(G_2)\right\}.
    \end{equation}
    If the graph-level readout is injective over node-representation multisets, then the graph representations of $G_1$ and $G_2$ are different. Therefore, PolyGIN distinguishes every graph pair distinguished by 1-WL within L iterations.
\end{proof}
Combining Proposition 1 and Theorem 2 yields the following corollary.

\textbf{Corollary 1.}

Under the injective-update and injective-readout conditions in Theorem 2, PolyGIN has exactly the same graph discriminative power as the 1-WL test. Without these injectivity conditions, PolyGIN remains upper bounded by 1-WL but may be strictly less powerful.

\subsection{Discussion: Polynomial Exactness Does Not Imply Injectivity}

Theorem 2 should be read as a conditional expressivity result. It does not state that every parameterization of PolyGIN is maximally expressive. Indeed, no such unconditional claim is possible.

The polynomial design of PolyGIN was introduced for a different purpose: to ensure that the scalar model output remains a bounded-degree multivariate polynomial in the input features. This property allows the derivative along the Aumann–Shapley path to have a known finite degree, which in turn enables exact Gauss–Legendre quadrature. The degree argument controls the algebraic form of the model output; it does not automatically control whether two different graph neighborhoods are mapped to two different hidden representations.

A simple example illustrates the distinction. If $W_2^{(\ell)}=0$ in some layer, then $P_\ell$ maps every aggregated vector to the zero vector. The layer is polynomial and satisfies the degree bound, but it is clearly non-injective and destroys all neighborhood information. More generally, if $W_1^{(\ell)}$, $W_2^{(\ell)}$, or the polynomial activation parameters collapse two distinct aggregation values, then the resulting PolyGIN layer cannot simulate the injective hashing step of WL.

Therefore, PolyGIN has two logically separate guarantees:
\begin{itemize}
    \item 1. Algebraic exactness for attribution.
    Because aggregation, linear projection, PolyScaleNorm, and PolyActivation are polynomial or linear operations, the model output has a bounded polynomial degree. This yields an exact finite Gauss–Legendre rule for Aumann–Shapley attribution under the stated assumptions.
    \item 2. WL-level discriminative power under injectivity.
    If the effective layer-wise update and the graph-level readout are injective over the relevant multisets, then PolyGIN matches the expressive power of 1-WL.
\end{itemize}
These two properties are compatible but not equivalent. The first is guaranteed by the polynomial architecture. The second requires the polynomial transformations to preserve the distinctness of WL refinement pairs.

\section{Equivalence Between Completeness Error and Quadrature Error}
\label{app:truncation_error}

This section shows that the completeness error in Eq.~\eqref{eq:completeness_error}, also referred to as the efficiency gap in some attribution literature, is equivalent to the quadrature error of the numerical path integral for Riemann-sum Integrated Gradients. This equivalence justifies using completeness error as an empirical proxy for the truncation error introduced by numerical path integration.

\begin{theorem}
\label{thm:efficiency_gap}
Let $F:\mathbb{R}^n\rightarrow\mathbb{R}$ be continuously differentiable. For an input $X\in\mathbb{R}^n$, a baseline $X'\in\mathbb{R}^n$, and the straight-line path $\gamma(\alpha)=X'+\alpha(X-X')$, the completeness error of an $m$-step Riemann approximation to Integrated Gradients equals the absolute quadrature error of the corresponding path integral.
\end{theorem}

\begin{proof}
The exact Aumann--Shapley attribution for feature $i$ is
\begin{equation}
    \phi_i(X)=(x_i-x_i')\int_0^1 \frac{\partial F(\gamma(\alpha))}{\partial x_i}\,d\alpha.
\end{equation}
Summing over all features and applying the chain rule gives
\begin{align}
    \sum_{i=1}^n\phi_i(X)
    &= \int_0^1 \sum_{i=1}^n \frac{\partial F(\gamma(\alpha))}{\partial x_i}(x_i-x_i')\,d\alpha \\
    &= \int_0^1 \frac{dF(\gamma(\alpha))}{d\alpha}\,d\alpha \\
    &= F(X)-F(X').
\end{align}
An $m$-step right Riemann approximation gives
\begin{equation}
    \sum_{i=1}^n\hat{\phi}_i(X)
    = \frac{1}{m}\sum_{k=1}^{m}\frac{dF(\gamma(k/m))}{d\alpha}.
\end{equation}
The quadrature error of this approximation is
\begin{equation}
    R_m = \frac{1}{m}\sum_{k=1}^{m}\frac{dF(\gamma(k/m))}{d\alpha}
    - \int_0^1\frac{dF(\gamma(\alpha))}{d\alpha}\,d\alpha.
\end{equation}
Substituting the expressions above yields
\begin{equation}
    R_m = \sum_{i=1}^n\hat{\phi}_i(X)-\big(F(X)-F(X')\big).
\end{equation}
Taking the absolute value gives exactly the completeness error:
\begin{equation}
    |R_m|=\left|\sum_{i=1}^n\hat{\phi}_i(X)-\big(F(X)-F(X')\big)\right|.
\end{equation}
Thus the empirical completeness error measures the quadrature error of the numerical path integral.
\end{proof}

\section{Disentangling Architecture and Attribution Algorithm}
\label{app:disentangle}

A possible confounding factor is that the explanation performance may arise from the PolyGIN architecture alone rather than from the explanation mechanism. To control for this factor, we evaluate multiple explainers on the same trained PolyGIN models. In this setting, all methods receive the same underlying predictive model and differ only in how they construct their explanations.

\begin{figure*}[h]
    \centering
    \IfFileExists{figure/F+_poly.pdf}{%
        \includegraphics[width=1.0\textwidth]{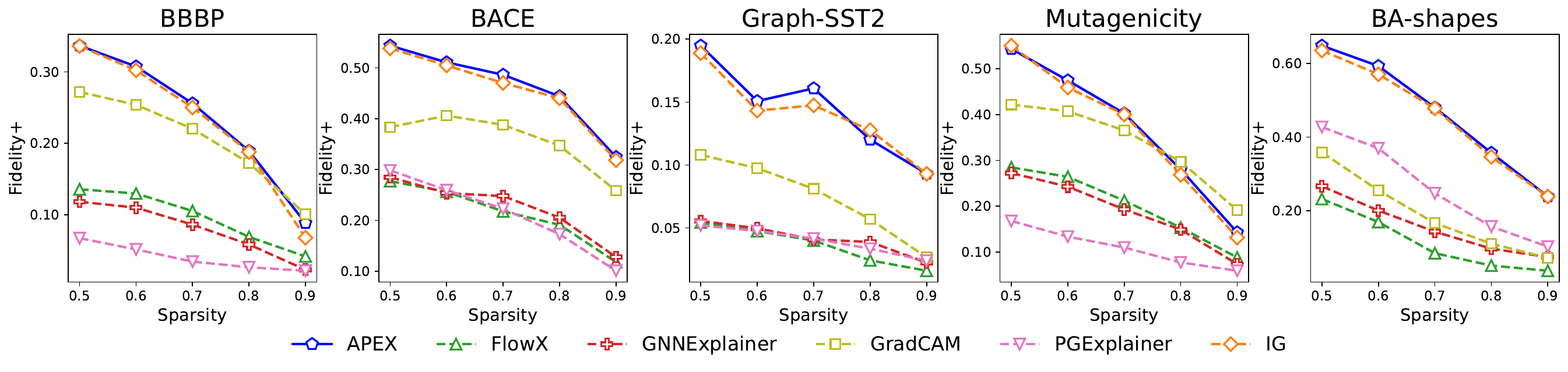}%
    }{%
        \fbox{\parbox{0.92\linewidth}{\centering Placeholder for Fidelity$^+$ comparisons where all explainers are evaluated on PolyGIN.}}%
    }
    \caption{Fidelity$^+$ comparison when all explainers are evaluated on the same PolyGIN models.}
    \label{fig:polygin_fid_plus}
\end{figure*}

\begin{figure*}[h]
    \centering
    \IfFileExists{figure/F-_poly.pdf}{%
        \includegraphics[width=1.0\textwidth]{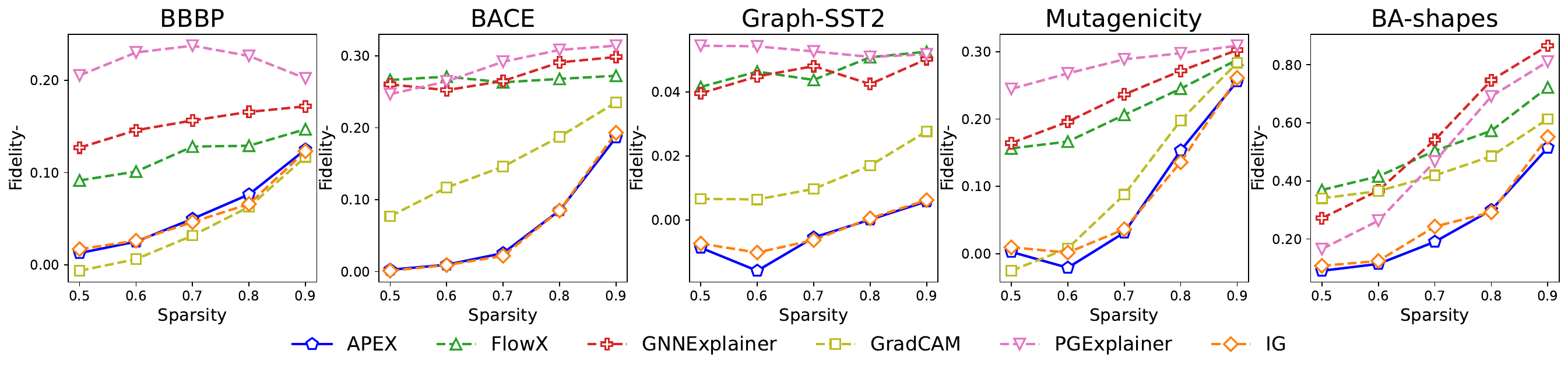}%
    }{%
        \fbox{\parbox{0.92\linewidth}{\centering Placeholder for Fidelity$^-$ comparisons where all explainers are evaluated on PolyGIN.}}%
    }
    \caption{Fidelity$^-$ comparison when all explainers are evaluated on the same PolyGIN models.}
    \label{fig:polygin_fid_minus}
\end{figure*}

Figures~\ref{fig:polygin_fid_plus} and~\ref{fig:polygin_fid_minus} show that numerical Integrated Gradients achieves fidelity results nearly identical to those of APEX when both are applied to PolyGIN. This agreement is theoretically expected because APEX and numerical IG target the same Aumann--Shapley attribution, although they evaluate the corresponding path integral differently. The close empirical agreement therefore provides a consistency check that APEX recovers the intended continuous path attribution rather than defining a different explanation quantity. In contrast, the other explanation methods do not exhibit substantial improvements when applied to the same PolyGIN models, and several methods even show degraded fidelity performance. This indicates that the polynomial architecture alone does not necessarily improve arbitrary post-hoc explainers. Rather, the proposed contribution is the specific co-design of a polynomial predictive architecture and its compatible exact path-attribution procedure.APEX does not introduce an attribution quantity different from Aumann–Shapley attribution. Instead, it jointly specifies an attribution-compatible predictive architecture and an architecture-certified procedure for evaluating the corresponding continuous attribution exactly within a finite deterministic budget. The same-backbone agreement between APEX and numerical IG validates the target attribution recovered by the framework; it does not imply that the framework-level fidelity can be causally assigned to the architecture alone, nor does it treat the architecture and attribution procedure as interchangeable standalone contributions. Accordingly, the results should be interpreted as evidence for the complete PolyGIN–exact-attribution co-design.

\section{Runtime Analysis}

We further compare the runtime required to generate one explanation. State-of-the-art learning-based explainers such as GNNExplainer can be costly because they optimize an instance-specific mask for each instance \cite{ying2019gnnexplainer}. Numerical IG avoids mask optimization but requires many path samples to reduce quadrature error \cite{sundararajan2017axiomatic}. APEX requires only the number of Gauss--Legendre points determined by the degree bound. In our 4-block PolyGIN setting, this corresponds to eight forward-backward evaluations. As shown in Figure~\ref{fig:time_comparison}, APEX is substantially faster than dense numerical integration methods while maintaining exact quadrature under the PolyGIN assumptions. GradCAM-like gradient methods may use fewer backward passes, but their fidelity results in Section~\ref{sec:fidelity} indicate that this lower cost comes with weaker attribution quality under the evaluated settings. Overall, APEX targets a practical trade-off: exact path-integral attribution for the polynomial model with a small, deterministic evaluation budget.

\label{app:time}

\begin{figure}[h]
    \centering
    \IfFileExists{figure/time.pdf}{%
        \includegraphics[width=0.62\linewidth]{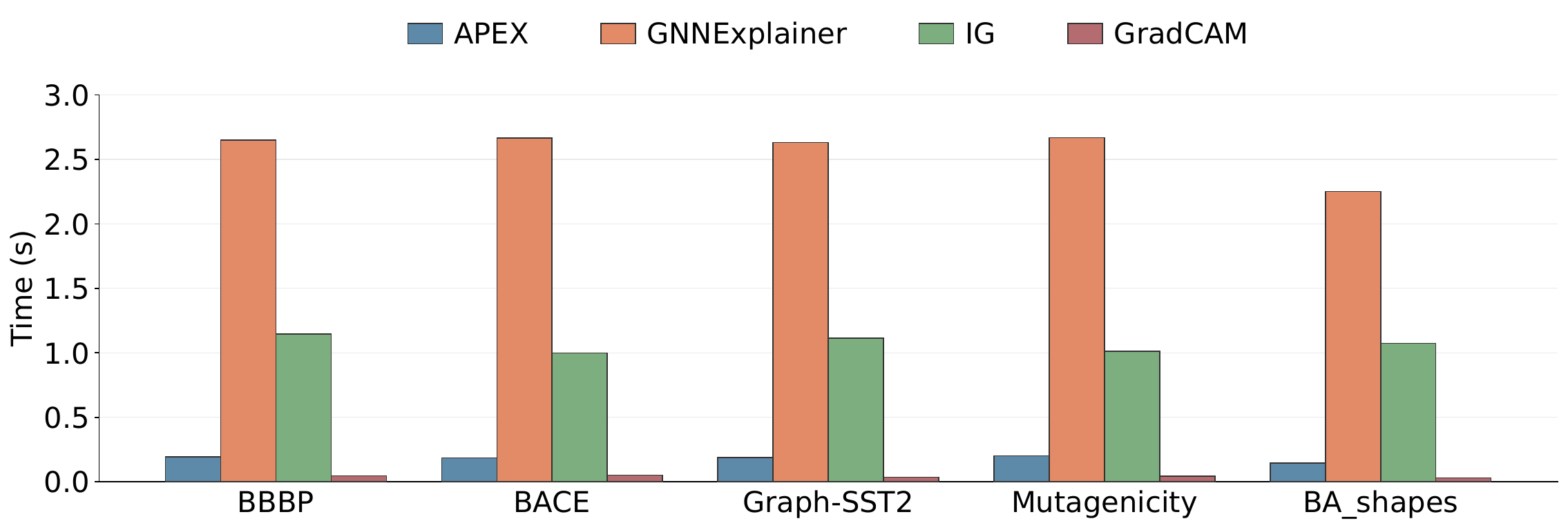}%
    }{%
        \fbox{\parbox{0.58\linewidth}{\centering Placeholder for runtime comparison across explanation methods.}}%
    }
    \caption{Average runtime for generating one graph explanation.}
    \label{fig:time_comparison}
\end{figure}

\clearpage

\begin{figure*}[h]
    \centering
    \includegraphics[width=1.0\linewidth]{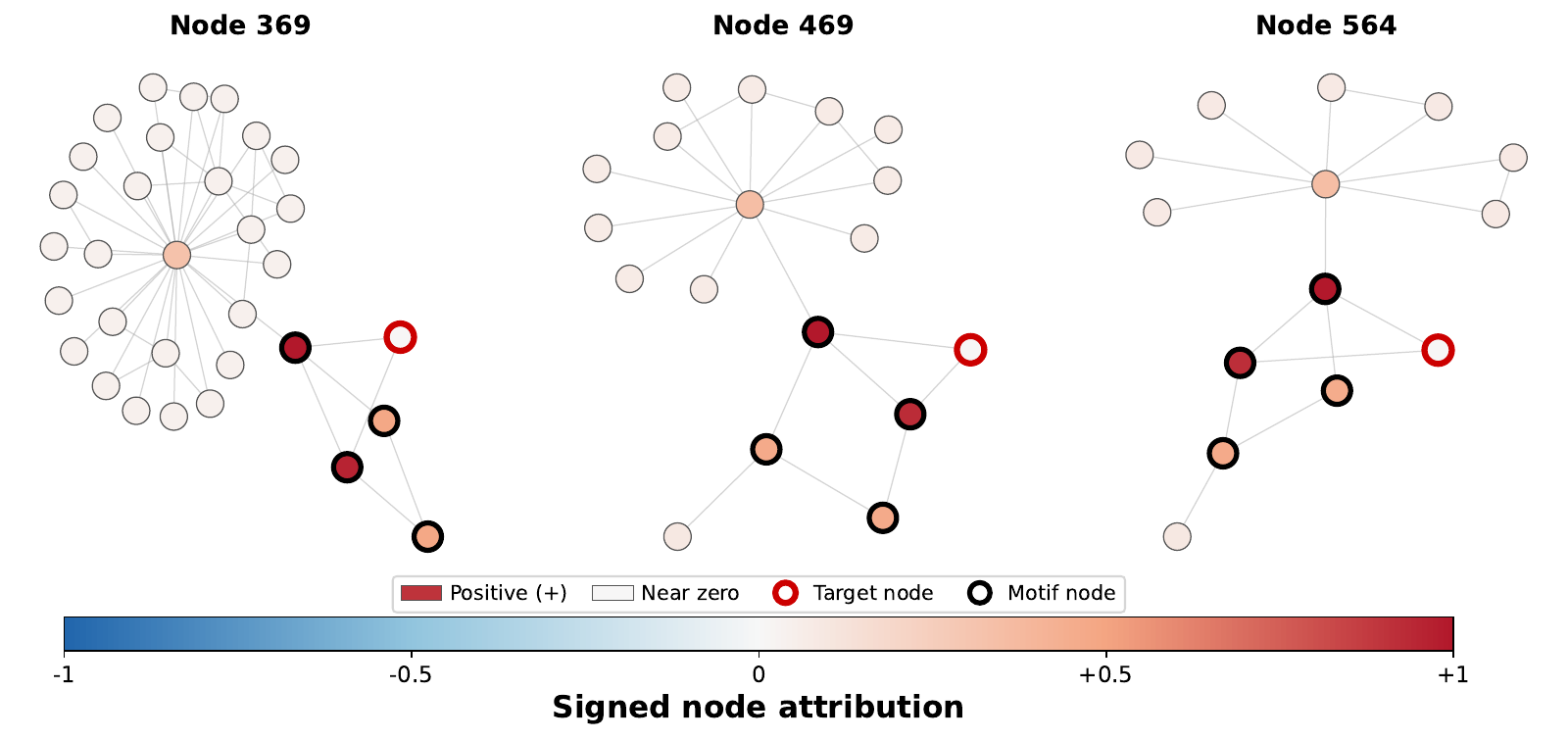}
    \caption{Qualitative visualization of APEX signed node attributions on BA-Shapes. 
The target node is marked by a red boundary, and motif nodes are marked by black boundaries. 
Red indicates positive contribution to the target-class logit, blue indicates negative contribution, and white indicates near-zero contribution. Across different target nodes, APEX consistently assigns positive evidence to the complete house motif while leaving most BA-backbone nodes nearly neutral, matching the ground-truth structural rule of the dataset.
}
    \label{fig:app_ba_shapes}
\end{figure*}

\begin{figure}[h]
    \centering
    \includegraphics[width=1.0\linewidth]{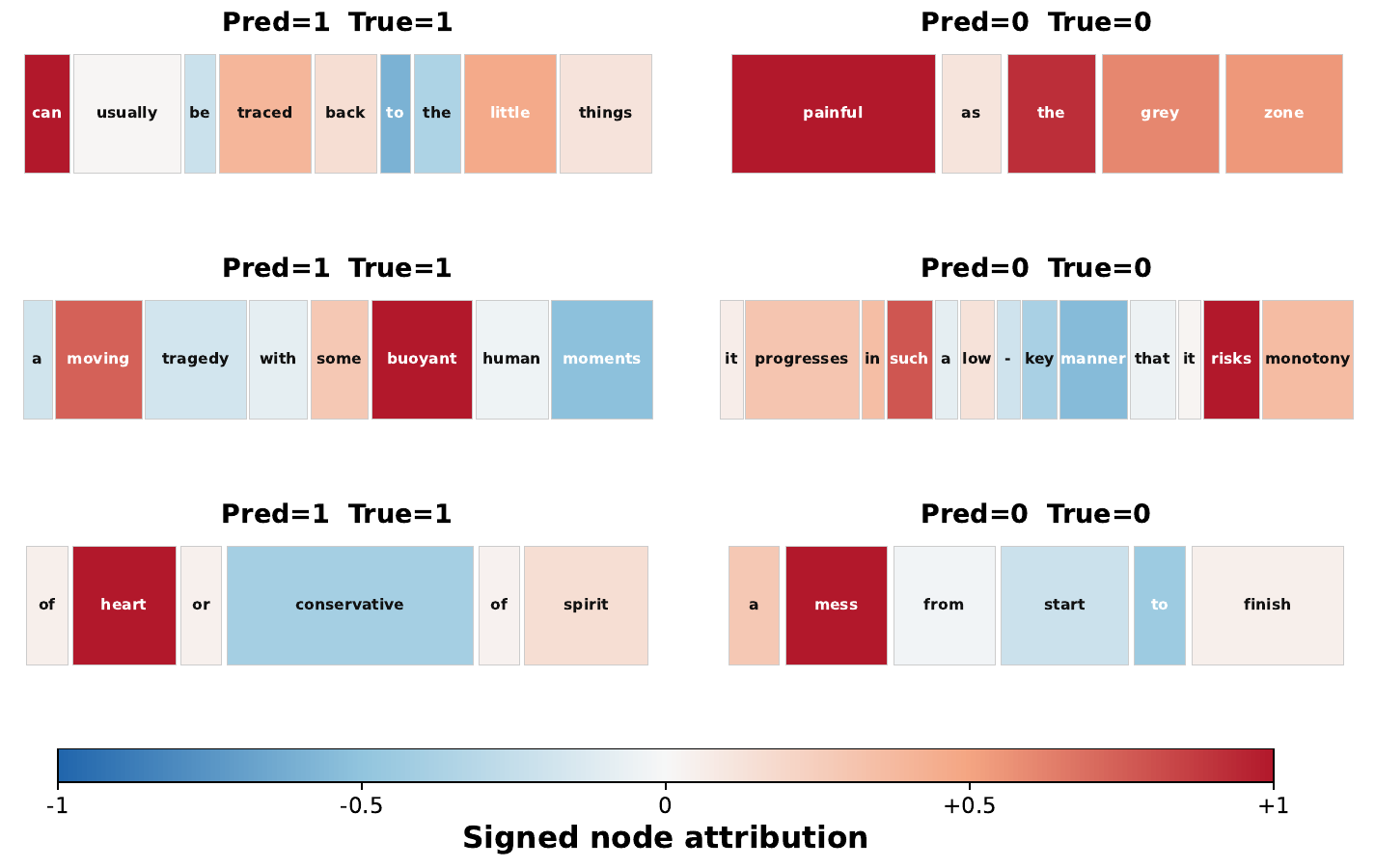}
    \caption{Token-level signed attribution visualizations of APEX on Graph-SST2. All examples shown are correctly classified, with labels 1 and 0 denoting positive and negative sentiment, respectively. Token colors indicate the direction and relative magnitude of each token’s contribution to the logit of the predicted class: positive values support the prediction, whereas negative values suppress it. Attribution scores are normalized to ([-1, 1]).}
    \label{fig:app_sst2}
\end{figure}

\begin{figure}[h]
    \centering
    \includegraphics[width=1.0\linewidth]{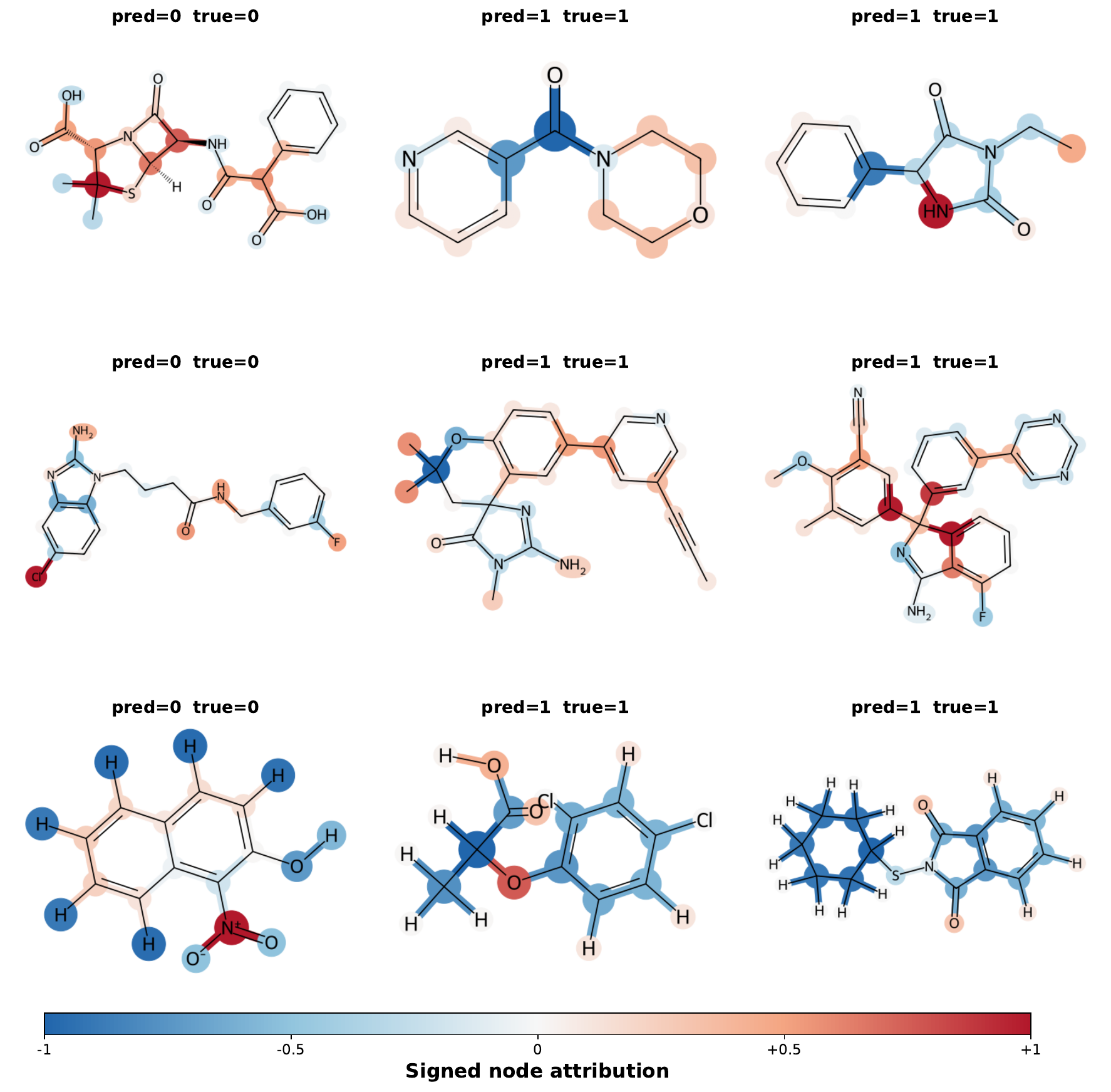}
    \caption{Signed atom-level attributions on molecular property prediction datasets. 
Rows correspond to BBBP, BACE, and Mutagenicity examples, respectively. 
Red atoms support the displayed target class, blue atoms oppose it, and white atoms have near-zero contribution. APEX highlights chemically meaningful evidence, including BBB-relevant polarity and scaffold effects, BACE-related hydrophobic pharmacophore fragments, and mutagenicity-associated structural alerts such as nitro groups. These examples demonstrate that APEX provides direction-aware molecular explanations rather than merely identifying important atoms.}
    \label{fig:app_molecules}
\end{figure}

\section{Additional Qualitative Results}
\label{app:qualitative}

\paragraph{Signed attribution visualization.}
We provide additional qualitative examples to illustrate how APEX explains graph predictions through signed node-level attributions. Many mask- and  subgraph-based explainers used in our comparison produce non-negative selection or importance scores, which typically assign non-negative importance scores to nodes, edges, or subgraphs, APEX directly decomposes the target-class prediction into positive and negative evidence. In our visualization, red nodes have positive attribution and therefore support the displayed target class, while blue nodes have negative attribution and therefore suppress it. White nodes have near-zero contribution. All scores are normalized to $[-1,1]$ for visualization only.

This distinction is crucial: an important node is not necessarily supportive. A token, atom, or structural component may be highly relevant precisely because it argues against the predicted class. APEX’s positive and negative node attributions are grounded in a rigorous Aumann–Shapley axiomatic foundation and constitute a complete signed decomposition of changes in the model output.

\paragraph{BA-Shapes.}
Figure~\ref{fig:app_ba_shapes} presents explanations on BA-Shapes,
a synthetic node-classification benchmark constructed by attaching
five-node house motifs to a Barabási--Albert backbone. Node labels
encode distinct structural roles within the resulting graph. For each
example, we visualize the local ego graph of the target node. The
target node, which is itself part of the house motif, is outlined in
red, while the remaining nodes of the corresponding ground-truth
motif are outlined in black.

Across all three examples, APEX clearly separates the class-defining
house structure from the surrounding BA backbone. The four non-target
nodes of the house motif consistently receive positive attribution,
with the strongest positive evidence concentrated on a subset of
structurally informative motif nodes. In contrast, the target node
itself has a near-zero attribution in these examples. Most nodes in
the surrounding BA backbone also receive negligible attribution,
although the backbone node at the motif attachment point exhibits a
small positive contribution.

These explanations closely match the construction of BA-Shapes:
the prediction is primarily supported by the target node's local
house structure rather than by the large high-degree neighborhood of
the BA backbone. In particular, APEX does not simply select the most
densely connected region. It localizes the positive evidence to the
ground-truth motif while assigning substantially smaller contributions
to structurally irrelevant backbone nodes. The signed visualization
therefore reveals both where the model obtains its evidence and the
direction in which each node affects the target-class score.

\paragraph{Graph-SST2.}
Figure~\ref{fig:app_sst2} shows token-level signed attributions on Graph-SST2. Each sentence is represented as a graph, and APEX assigns a signed contribution to each token with respect to the displayed predicted class. Since all shown examples are correctly classified, the explanations can be inspected directly against human sentiment intuition.

For positive-sentiment predictions, APEX highlights tokens such as ``can'', ``moving'', ``buoyant'', and ``heart'' as positive evidence, while words with negative or contrastive semantics, such as ``tragedy'' and ``conservative'', receive negative attribution. Conversely, for negative-sentiment predictions, tokens such as ``painful'', ``mess'', and ``risks'' are assigned positive attribution because they support the negative class. Other tokens, including neutral syntactic connectors or semantically opposing words, receive weak or negative contributions.

These examples demonstrate a qualitative advantage of signed attribution over ordinary saliency ranking. A non-negative explanation can identify which words are important, but it cannot distinguish whether an important word supports or contradicts the predicted class. APEX separates these two cases explicitly. This is especially useful for sentiment analysis, where sentences often contain mixed evidence, contrastive phrases, or locally negative words inside globally positive sentences.

\paragraph{Molecular graphs.}
Figure~\ref{fig:app_molecules} visualizes signed atom-level attributions on molecular property prediction tasks. We consider examples from BBBP, BACE, and Mutagenicity. The label semantics used in the visualization are summarized in Table~\ref{tab:label_semantics}. The examples demonstrate that APEX resolves a molecular prediction into spatially localized and directionally opposing structural evidence, rather than producing a direction-agnostic ranking of important atoms.

\begin{table}[h]
\centering
\caption{Label semantics used in the molecular visualization examples.}
\vspace{0.5em}
\label{tab:label_semantics}
\begin{tabular}{lll}
\toprule
Dataset & Label $1$ & Label $0$ \\
\midrule
BBBP & BBB-permeable & BBB-impermeable \\
BACE & Active inhibitor & Inactive molecule \\
Mutagenicity & Non-mutagenic & Mutagenic \\
\bottomrule
\end{tabular}
\end{table}

The BBBP examples show that APEX represents blood--brain barrier permeability as a balance of competing local structural contributions. For the BBB-impermeable molecule in the first column, positive evidence is concentrated on selected atoms of the sulfur-containing fused core and the adjacent carbon framework. By contrast, several carbonyl, hydroxyl, and other oxygen-containing sites contribute negatively to the impermeable-class score. The explanation therefore identifies a specific core region as the dominant source of the model's impermeability prediction, rather than assigning uniformly positive importance to all polar atoms in the molecule. For the first BBB-permeable molecule, the carbonyl-centered linker and its adjacent attachment atom provide the strongest negative evidence, whereas the carbon atoms in the right-hand heterocyclic ring contribute positively to the permeable class. The second permeable example exhibits an equally clear competition: the imide-like NH is the strongest positive contributor, and the terminal alkyl substituent also supports permeability, while the two carbonyl regions and part of the aryl--heterocycle connection oppose the prediction. Taken together, these examples show that APEX does not reduce BBB permeability to a single atom type or global polarity rule. Instead, it distinguishes the contributions of chemically different local environments within the same molecule. The negative contributions of carbonyl-rich regions in the two permeable examples are consistent with the permeability cost commonly associated with exposed polarity and hydrogen-bonding capacity, while the strong positive contribution of the NH group in the third molecule illustrates that the learned evidence remains strongly context dependent.

The BACE examples reveal class evidence distributed across coordinated scaffold and substituent regions. In the inactive molecule, the terminal chlorine atom provides the strongest positive evidence for the inactive class. Additional positive contributions arise from the amide-containing linker and the terminal fluorinated region, whereas most of the fused nitrogen-containing heteroaromatic scaffold provides negative evidence. The model therefore distinguishes between opposing contributions within the same molecule rather than treating the presence of aromatic rings, heteroatoms, or halogens as uniformly indicative of activity. For the first active molecule, positive evidence extends along the central phenyl--heteroaryl axis and into the terminal alkyne-containing substituent. In contrast, the oxygenated branched region on the left and several atoms in the lower heterocyclic ring oppose the active-class prediction. In the second active molecule, APEX assigns the dominant positive evidence to the central multiring scaffold and its ring-junction region, while several peripheral nitrogen-containing, fluorinated, and heteroaromatic sites contribute negatively. The two active examples thus exhibit a common organization: activity is supported by a coherent central scaffold, while peripheral functional groups modulate the prediction in both directions. This pattern is consistent with the multi-subsite nature of BACE-1 ligand recognition, for which activity depends on the coordinated arrangement of aromatic, hydrophobic, and polar functionalities rather than on the presence of any single fragment. APEX captures this organization as a signed structure--activity decomposition, identifying both the substructures that drive the active prediction and those that weaken it.

The Mutagenicity examples illustrate three distinct forms of signed molecular evidence. In the mutagenic nitroaromatic molecule, the dominant positive attribution is sharply localized at the nitro nitrogen. The nitro oxygens, the hydroxyl-containing region, and most peripheral hydrogen atoms contribute negatively to the mutagenic-class score. APEX therefore isolates the nitro center as the principal class-supportive site instead of assigning uniformly positive importance to the entire aromatic molecule. This localization recovers a well-established structural alert for mutagenicity and, at the same time, exposes substantial counter-evidence elsewhere in the molecule. The first non-mutagenic molecule exhibits the opposite organization. A single ether oxygen provides the strongest positive evidence for the non-mutagenic class, with a smaller positive contribution near the hydroxyl-containing terminus, whereas most of the surrounding carbon framework contributes negatively. The final non-mutagenic molecule is particularly informative: nearly all displayed heavy-atom attributions are negative or close to zero, despite the correct class-1 prediction. This result directly reflects the reference-based nature of signed attribution. The molecular structure decreases the target-class score relative to the chosen reference, while the final score nevertheless remains sufficient for the model to predict the non-mutagenic class.These examples demonstrate that APEX can distinguish a localized mutagenicity alert, a localized contribution supporting the non-mutagenic class, and a prediction dominated by negative reference-relative evidence. Such cases cannot be represented faithfully by a non-negative atomic importance map, which would obscure both the direction of the evidence and the qualitatively different roles played by the highlighted atoms.

In general, for examples of Mutagenicity, the molecular examples show that APEX produces chemically structured signed explanations: it localizes class-supportive substructures, identifies counter-evidence within the same molecule, and exposes predictions whose interpretation depends critically on the reference output. The observed patterns agree with established structure--property or structure--activity knowledge in several cases, but the attributions describe the evidence used by the trained model rather than experimentally validated molecular mechanisms.

Overall, these qualitative results show that APEX yields explanations that are structurally, linguistically, and chemically interpretable across very different graph domains. The examples should not be interpreted as definitive causal mechanisms; rather, they serve as sanity checks that the learned signed decomposition aligns with known dataset rules, human sentiment intuition, and domain-specific molecular knowledge.

\end{document}